\documentclass[preprint,14pt]{elsarticle}
\usepackage{amssymb}
\usepackage{amsmath}
\usepackage{subfigure}
\usepackage{float}
\usepackage{xcolor}
\usepackage{hyperref} 
\hypersetup{
colorlinks=true,
linkcolor=blue,
anchorcolor=blue,
citecolor=blue
}
\usepackage{booktabs}
\usepackage{comment}
\usepackage{caption}
\usepackage{subcaption} 

\usepackage{amsthm}
\usepackage{commutative-diagrams}
\allowdisplaybreaks[1]

\theoremstyle{definition}
\newtheorem{thm}{Theorem}
\newtheorem{rmk}[thm]{Remark}

\newtheorem{lem}[thm]{Lemma}

\newtheorem{defi}[thm]{Definition}

\newtheorem{Question}[]{Question}

\newcommand{\V}{\mathbb{V}}
\newcommand{\J}{\mathcal{J}}
\newcommand{\I}{\mathcal{I}}
\newcommand{\F}{\mathcal{F}}
\newcommand{\K}{\mathcal{K}}
\newcommand{\IM}{\mathcal{I}_5}

\begin{document}

\begin{frontmatter}



\title{Geometric Stratification for Singular Configurations of the P3P Problem via Local Dual Space} 

 \author[label1]{Xueying Sun}
 \author[label2,label3]{Zijia Li}
 \author[label1,label4]{Nan Li}
 \affiliation[label1]{organization={School of Mathematical Sciences, Shenzhen University},
             country={China}}

\affiliation[label2]{organization={State Key Laboratory of Mathematical Sciences},
             country={China}}

\affiliation[label3]{organization={University of Chinese Academy of Sciences},
             country={China}}

\affiliation[label4]{organization={Guangdong Provincial Key Laboratory of Intelligent Information Processing},
             country={China}}

\begin{abstract}
This paper investigates singular configurations of the P3P problem. Using local dual space, a systematic algebraic-computational framework is proposed to give a complete geometric stratification for the P3P singular configurations with respect to the multiplicity $\mu$ of the camera center $O$: for $\mu\ge 2$, $O$ lies on the ``danger cylinder'', for $\mu\ge 3$, $O$ lies on one of three generatrices of the danger cylinder associated with the first Morley triangle or the circumcircle, and for $\mu\ge 4$, $O$ lies on the circumcircle which indeed corresponds to infinite P3P solutions. Furthermore, a geometric stratification for the complementary configuration $O^\prime$ associated with a singular configuration $O$ is studied as well: for $\mu\ge 2$, $O^\prime$ lies on a deltoidal surface associated with the danger cylinder, and for $\mu\ge 3$, $O^\prime$ lies on one of three cuspidal curves of the deltoidal surface.
\end{abstract}

\begin{graphicalabstract}
\begin{figure}[H]
\centering  
\subfigure[n=4]{
\includegraphics[height= 0.45\linewidth]{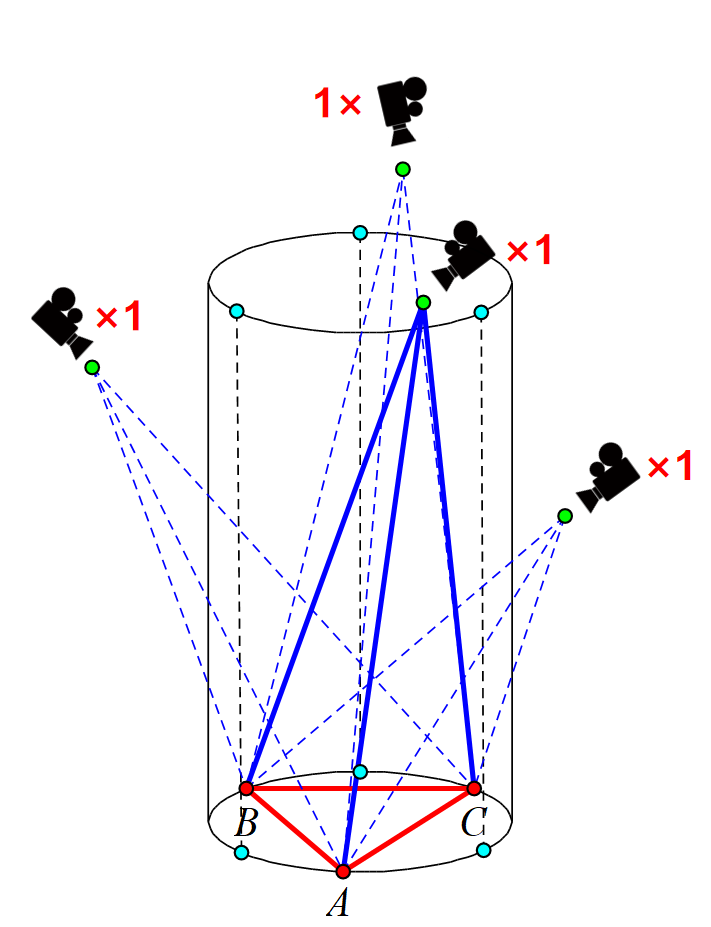}}
\subfigure[n=3]{
\includegraphics[height= 0.45\linewidth]{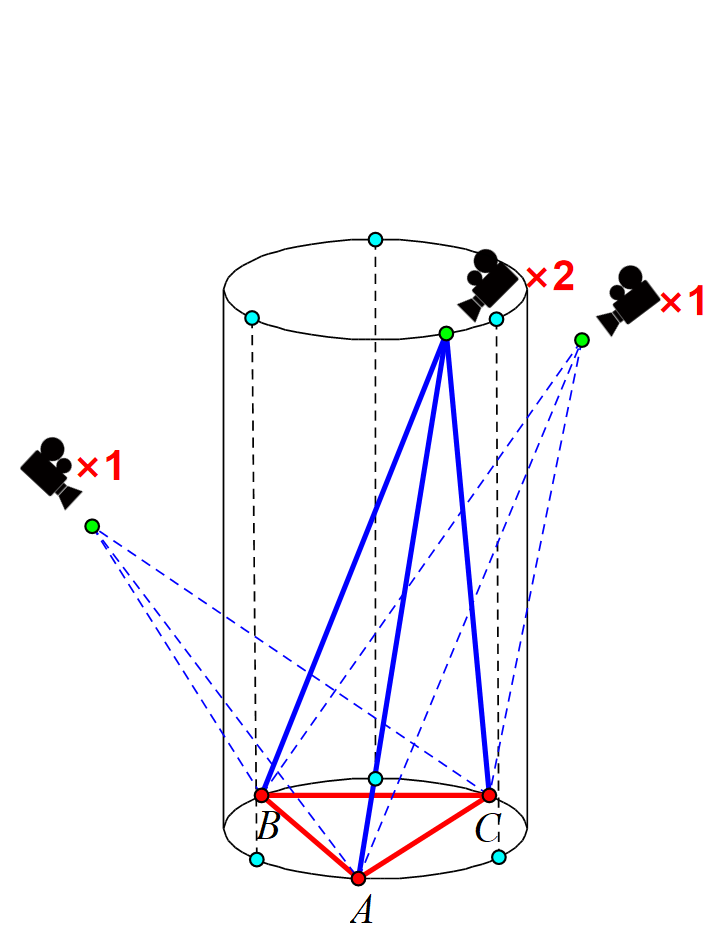}} \\
\subfigure[n=2]{
\includegraphics[height= 0.45\linewidth]{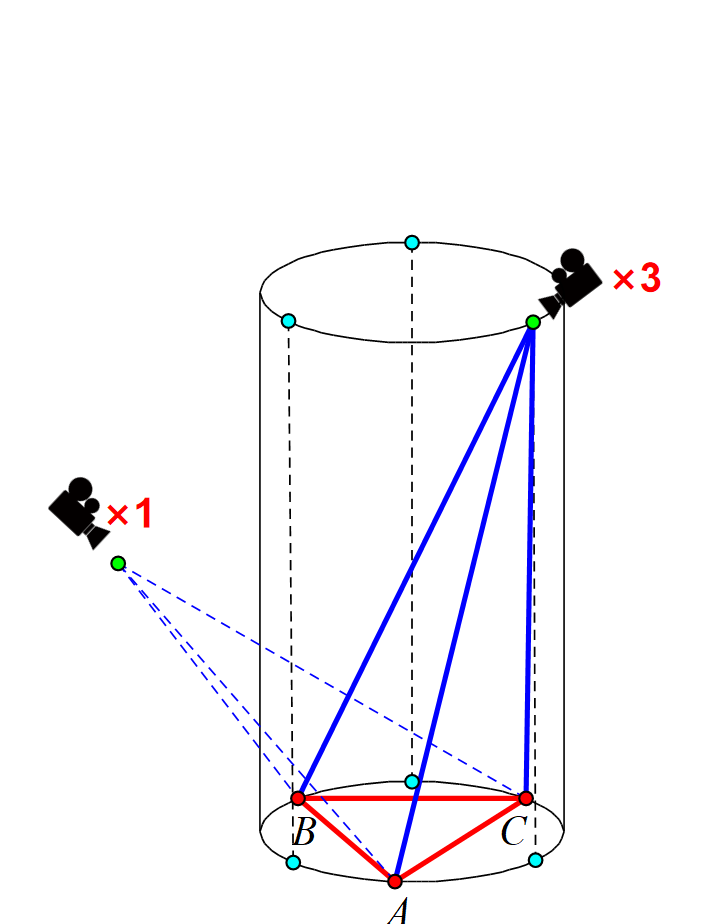}}
\hspace{10pt}
\subfigure[n=$\infty$]{
\includegraphics[height= 0.45\linewidth]{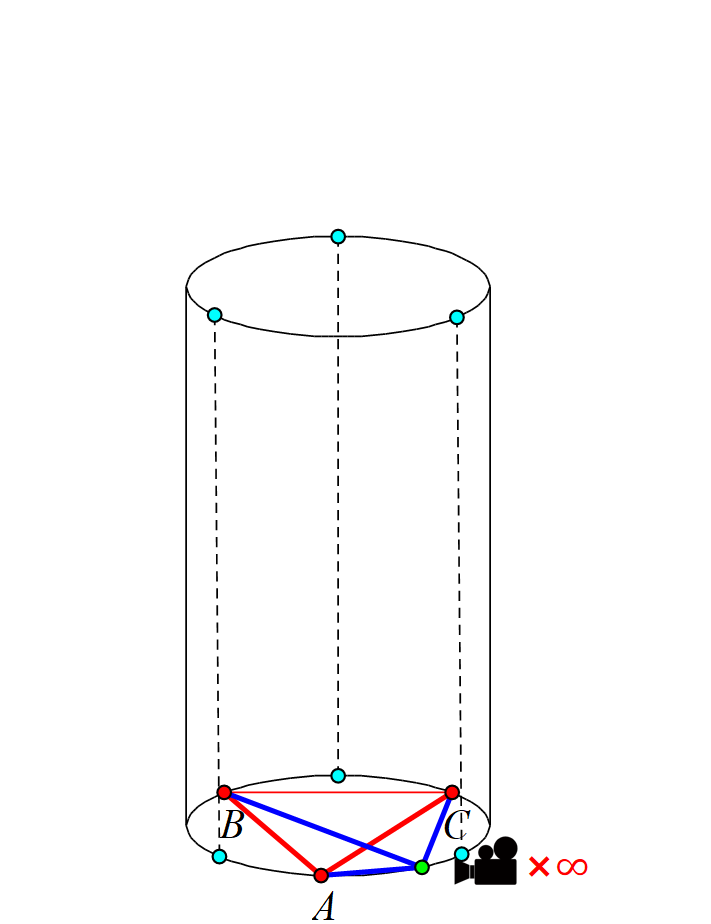}}
\caption{Multi-solution phenomenon and singular configurations of the P3P problem. In different configurations, the number of distinct solutions of the P3P equation system varies (up to $4$ or infinite). Here $n$ denotes the number of distinct solutions and the {\color{red}red} number denotes the multiplicity of the solution. In this paper, we investigate the geometry of singular P3P solutions as well as the geometry of complementary solutions with respect to different multiplicities. }
\end{figure}

\end{graphicalabstract}

\begin{highlights}
\item Using local dual space, a systematic 
stratification framework is proposed for collecting geometric conditions of singular P3P solutions of different multiplicities.
\item By symbolic computation, a complete geometric stratification for the P3P singular solutions w.r.t. multiplicity is given.
\item By case study, a geometric stratification for the complementary solutions associated with  singular P3P solutions w.r.t. multiplicity is given.
\end{highlights}


\end{frontmatter}



\section{Introduction}
 The \emph{Perspective-Three-Point (P3P)} problem, to determine the position and orientation of a calibrated camera from three 3D-2D point correspondences, is a classical problem in computer vision with many applications in camera localization, augmented reality, robotics navigation, and 3D reconstruction. It was first introduced by Grunert in 1841 \cite{grunert1841pothenotische} and has been popular since 1981 \cite{fischler1981random}, because P3P served as the minimal problem of PnP in the RANSAC framework (RANdom SAmple Consensus).

The P3P problem can have a varying number of solutions (up to $4$ or infinite) in different configurations between three 3D points and the camera center. When the number of solutions is strictly smaller than $4$ or infinite in certain configurations, i.e., \emph{singular configuration}, the P3P solver usually becomes inaccurate and unstable \cite{10203140,pascual2021complete}. In fact, the study on the multi-solution phenomenon of the P3P problem has been an active topic since its very inception.

From an algebraic perspective, Gao et al. \cite{gao2003complete} used Wu-Ritt's zero decomposition approach to give a complete triangular decomposition for the P3P equation system, which is defined in (\ref{tetra_form}), and provided a complete algebraic classification for (\ref{tetra_form}) to have a different number of solutions. Faugère et al. \cite{faugere2008classification} computed at least one point in each connected component of the real discriminant variety of (\ref{tetra_form}), then provided another algebraic classification. However, these works did not give intuitive descriptions of the geometry of their classifications, especially for singular configurations, which are important to guide the arrangement of 3D points in real applications.

From a geometric perspective, Thompson \cite{thompson1966space} first connected the singular Jacobian determinant of (\ref{tetra_form}) with the \emph{danger cylinder} (see Figure~\ref{mu2}(a)), but its relation to the number of P3P solutions was not provided. Wolfe et al. \cite{wolfe1991perspective} gave a geometrically sufficient condition in which (\ref{tetra_form}) always had four distinct solutions.
Su et al. \cite{su1998necessary} proved that when the camera center lies on the circumcircle of the  3D triangle (see Figure~\ref{fig:mu3}(c)), (\ref{tetra_form}) always has infinite solutions.
 Zhang et al. \cite{zhang2006danger} proved that when the camera center lies on the danger cylinder but far from the circumcircle, the number of P3P solutions will be at most $3$ and one of them is at least a double zero.
Rieck \cite{rieck2024geometric} gave a geometrically necessary condition in which (\ref{tetra_form}) had one real solution. However, most geometric studies on the P3P multi-solution phenomenon are fragmented; a complete collection of geometric conditions for (\ref{tetra_form}) to have a different number of solutions is still lacking.


\emph{Local Dual Space} provides an approach to locally characterize singular solutions by higher-order differential operators \cite{LiZhi:2009,Mantzaflaris2014,MANTZAFLARIS2016114}, which has been successfully applied to certify approximations of isolated singular solutions of polynomial systems \cite{HJLZ2020,LZ:2011,Li2014Verified,Mantzaflaris2011issac,Mantzaflaris2020issac,MANTZAFLARIS2023223}. In this paper, we adopt a computational criterion in \cite{li2022improved} arising from the local dual space to propose a systematic 
 stratification framework for completely collecting geometric conditions with respect to the different multiplicity of singular P3P solutions.
Similar ideas of using higher-order derivatives to locally analyze kinematic singularities are reported in \cite{Mueller2014higher,Mueller2016local}.

\label{sec1}


In this work, we systematically investigate singular configurations of the P3P problem from a geometric perspective. We establish that
\begin{itemize}
    \item singular P3P solutions with multiplicity $\mu\ge 2$ are located on the danger cylinder, solutions with multiplicity $\mu\ge 3$ are located on one of three generatrices of the danger cylinder associated with the first Morley triangle or the circumcircle, and solutions with multiplicity $\mu\ge 4$ are located on the circumcircle whose multiplicity is indeed infinite.
    Namely, it is impossible to have a unique solution over the complex field.
    \item  other P3P solutions complementary to a $\mu\ge 2$ solution are located on a deltoidal surface of order $12$ which is tangent to the danger cylinder along three Morley generatrices and intersects the danger cylinder at the circumcircle, and solutions complementary to a $\mu\ge 3$ solution are located on three cuspidal curves of the deltoidal surface that are its own singular locus. In other words, it is impossible to have two double P3P solutions.
\end{itemize}

The rest of the paper is organized as follows. Section~\ref{sec2} makes necessary recalls on the formulation of the P3P equation system, the number of P3P solutions, the P3P singular configuration, and the local dual space in computational algebraic geometry. Section~\ref{geometric stratification} gives a complete geometric stratification for P3P singular configurations with respect to the multiplicity of the camera center. 
Section~\ref{Geometric Discriminant} uses a case study to demonstrate a complementary stratification for P3P singular configurations.
Section~\ref{conclusion} concludes the work. Some detailed proofs and extended calculations are included in Appendix.

\section{Preliminaries}\label{sec2}
\subsection{P3P problem}
Suppose that three 3D points $\mathbf{X}_i\in \mathbb{R}^3,i\in\{1,2,3\}$ are projected onto three normalized 2D points $\mathbf{x}_i\in\mathbb{P}^2,i\in\{1,2,3\}$ through an unknown calibrated camera centered at $O$ (see Figure~\ref{Fig.1}(a)), then the P3P problem is to find the orientation matrix $\mathbf{R}\in \mathrm{SO}(3)$ and the translation vector $\mathbf{t}\in \mathbb{R}^3$ such that $\mathbf{X}_i\leftrightarrow\mathbf{x}_i$ can be correspondingly related by
\begin{equation}\label{direct_form}
e_i\mathbf{x}_i=\mathbf{R}\mathbf{X}_i+\mathbf{t},~~i\in\{1,2,3\},
\end{equation}
where $e_i=||\mathbf{RX}_i+\mathbf{t}||$ is the distance between $O$ and $\mathbf{X}_i$ in the camera coordinates.
A simple count shows there are 9 equations in (\ref{direct_form}) with $9$ degrees of freedom in the unknowns ($3$ for $e_i$, $3$ for $\mathbf{R}$, $3$ for $\mathbf{t}$), which generally result in a finite number of common zeros.
\begin{figure}[H]
\centering  
\subfigure[]{
\includegraphics[height= 0.4\linewidth]{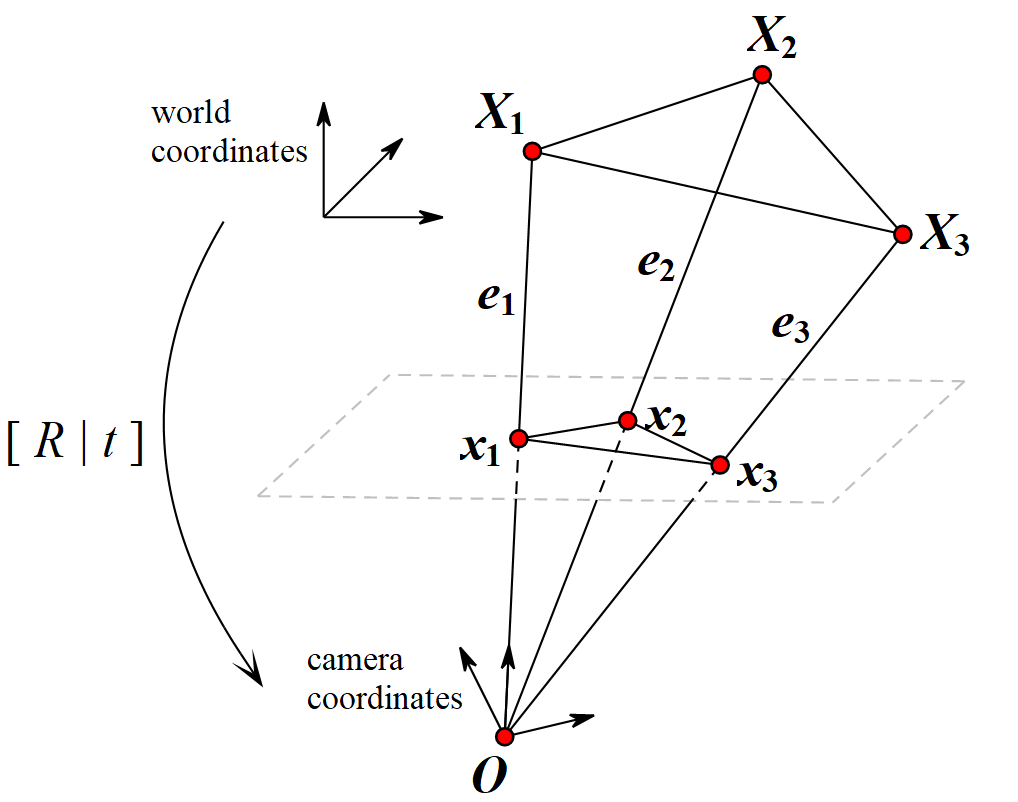}}
\subfigure[]{
\includegraphics[height= 0.4\linewidth]{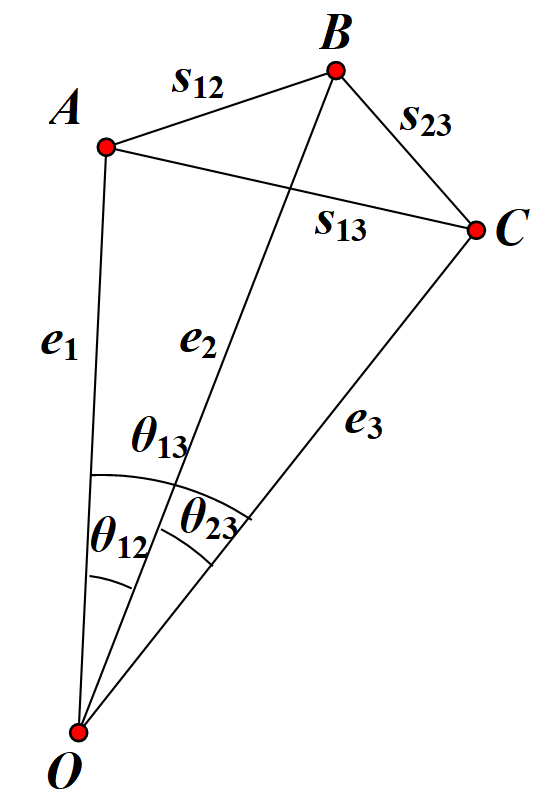}}
\caption{The geometry of the P3P problem.}
\label{Fig.1}
\end{figure}
As shown in Figure~\ref{Fig.1}(b), the system (\ref{direct_form}) can be reformulated into a simpler and more geometrically meaningful system. For $i,j\in\{1,2,3\}$ with $i<j$, define
\begin{itemize}
    \item $s_{ij}=||\mathbf{X}_i-\mathbf{X}_j||$, i.e., the length of the 3D triangle $\triangle ABC$,
    \item $\cos\theta_{ij}=\mathbf{x}_i^{\top}\mathbf{x}_j$, i.e., the angle between edges of the tetrahedron $O\triangle ABC$.
\end{itemize}
By subtracting different pairs of (\ref{direct_form}) and taking their norms, we can derive
\begin{equation}\label{tetra_form}
    e_i^2+e_j^2-2\cos\theta_{ij}\cdot e_ie_j=s_{ij}^2,~~i,j\in\{1,2,3\}~\&~i<j,
\end{equation}
which is actually the Law of Cosines in $\triangle OAB$, $\triangle OBC$ and $\triangle OAC$.

Once (\ref{tetra_form}) is solved, i.e., $(e_1,e_2,e_3)$ is determined, excluding some degenerate configuration (3pt colinearity or 4pt coplanarity), $\mathbf{R}\in \mathrm{SO}(3)$ can be uniquely recovered by
\begin{equation*}
\begin{aligned}
\mathbf{R}&=\mathbf{M}\cdot\left[\mathbf{X}_1-\mathbf{X}_2,\mathbf{X}_3-\mathbf{X}_1,(\mathbf{X}_1-\mathbf{X}_2)\times(\mathbf{X}_3-\mathbf{X}_1)\right]^{-1}, \\
\mathbf{M}&=[e_1\mathbf{x}_1-e_2\mathbf{x}_2,e_3\mathbf{x}_3-e_1\mathbf{x}_1,(e_1\mathbf{x}_1-e_2\mathbf{x}_2)\times(e_3\mathbf{x}_3-e_1\mathbf{x}_1)],
\end{aligned}
\end{equation*}
and $\mathbf{t}$ can be recovered by (\ref{direct_form}) straightforwardly. Hence, the number of solutions for the system (\ref{direct_form}) is the same as the number of solutions for (\ref{tetra_form}).

Clearly, there are exactly $3$ quadratic equations and $3$ unknowns in (\ref{tetra_form}), which generally result in $8$ common zeros (counting multiplicities) by B\'ezout's theorem. Moreover, since all monomials in (\ref{tetra_form}) are of even degree, there are at most $4$ distinct positive zeros ($(e_1,e_2,e_3)$ is a zero implies $-(e_1,e_2,e_3)$ is a zero).

In a general configuration, the system (\ref{tetra_form}) admits four distinct complex solutions (discarding central symmetry). However, the number of distinct solutions will change in certain configurations (singular solutions appear), which usually lead to numerical instability and inaccuracy for P3P solvers.

We are motivated to investigate such configurations.

\subsection{Singular Configuration}

For the P3P problem, the camera center $O$ is called a \emph{singular configuration} with respect to a 3D triangle $\triangle ABC$ if  $(|OA|,|OB|,|OC|)$ is a singular solution of (\ref{tetra_form}) with multiplicity greater than two.

Clearly, (\ref{tetra_form}) is a polynomial system by taking $\cos\theta_{ij}$ as constants, so we define isolated singular solutions of polynomial systems and their multiplicity.

\begin{defi}[\cite{li2022improved}]\label{mul}
Given a system $f \in \mathbb{C}[X]^n$ where $X=\{X_1, \ldots, X_n\}$, then $\xi \in \mathbb{C}^n$ is an isolated solution of $f$ with multiplicity $\mu$ if it satisfies that

\vspace{-15pt}

\begin{item}
\item[(1)] $\xi \in f^{-1}(\mathbf{0}):=\{x \in \mathbb{C}^n ~|~ f(x)=\mathbf{0}\}$,
\item[(2)] $\exists B(\xi, r):=\{x \in \mathbb{C}^n ~|~ \|x-\xi\|\leq r\}$ ($r>0$) such that $B(\xi, r) \cap f^{-1}(\mathbf{0}) = \{\xi\}$,
\item[(3)] $|B(\xi, r) \cap g^{-1}(\mathbf{0})| = \mu$ for $\forall g \in \mathbb{C}[X]^n$ that is sufficiently close to $f$.
\end{item}
\end{defi}

Now, we propose two crucial questions concerning singular configurations of the P3P problem.
Given a 3D triangle $\triangle ABC$, i.e., $(s_{12}, s_{13}, s_{23})$ is fixed, suppose $(e_1,e_2,e_3)$ is a singular solution of (\ref{tetra_form}) with multiplicity $\geq\mu$ $(\mu=2,3,4)$, we want to know


\begin{Question}\label{Q1}
    where is the location of the corresponding center $O$?
\end{Question}

\begin{Question}\label{Q2}
    where are the other centers complementary to $O$?
\end{Question}


Note that Gao et al. \cite{gao2003complete} provided a complete algebraic classification for (\ref{tetra_form}) to have a different number of solutions, i.e., collecting algebraic conditions on $s_{ij}$ and $\cos\theta_{ij}$. In contrast, this work aims to provide a complete geometric stratification for (\ref{tetra_form}) to have singular solutions with different multiplicity, i.e., collecting geometric conditions on $O$ with respect to $\triangle ABC$.
More precisely, for $\mu=2,3,4$, let $\mathcal{V}_{\geq \mu}$ denote the set of singular solutions of (\ref{tetra_form}) with multiplicity $\geq\mu$, we want to give a complete representation
\begin{equation}\label{stra}
\mathcal{V}_{\geq 2}\supseteq\mathcal{V}_{\geq 3}\supseteq\mathcal{V}_{\geq 4},
\end{equation}
which is called \emph{the geometric stratification for singular configurations} w.r.t. multiplicities.
Furthermore, suppose that $O\in\mathcal{V}_{\geq \mu}$ and $O^\prime$ is a complementary solution of (\ref{tetra_form}) associated with $O$, and let $\mathcal{V}^{\prime}_{\geq \mu}$ denote the set of complementary solutions with multiplicity $\geq\mu$, we want to give a complete representation
\begin{equation}\label{stra_star}
\mathcal{V}^{\prime}_{\geq 2}\supseteq\mathcal{V}^{\prime}_{\geq 3}\supseteq\mathcal{V}^{\prime}_{\geq 4},
\end{equation}
 which is called \emph{the complementary stratification for singular configurations} w.r.t. multiplicities.

 This work focuses on answering Questions~\ref{Q1} and~\ref{Q2} in the form of (\ref{stra}) and (\ref{stra_star}) respectively.

 \begin{rmk}
     Our geometric stratification represents an application of the general stratification theory, e.g., the canonical Whitney stratification \cite{henry1983limites, teissier1982multiplicites},  to a new context: instead of studying the singularities of a space, we stratify a smooth parameter space  (e.g., (\ref{tetra_form})) according to the solution behavior of a geometric system. This extends the utility of stratification beyond its traditional domain of singularities, making it a tool for parametric problems in computational geometry and algebraic kinematics.
 \end{rmk}

\subsection{Local dual space}

To formulate the geometric stratification within a computationally algebraic framework, we introduce the \emph{local dual space} of an ideal at its zero.

\begin{defi}[\cite{li2022improved}]
Given $f \in \mathbb{C}[X]^n$ and $\xi \in f^{-1}(\mathbf{0})$, the differential operator space defined by
\[\mathcal{D}_{f,\xi}=\{\Lambda\in\mathbb{C}[X]^{*}~|~\Lambda[g](\xi)=0,\forall g\in\langle f \rangle\},\]
is called the \emph{local dual space} of $\langle f \rangle$ at $\xi$.
Moreover, if $\xi$ is an isolated singular solution of $f$ with multiplicity $\mu$, then $\mu=\dim \mathcal{D}_{f,\xi}$.
\end{defi}

Note that existing geometric studies \cite{zhang2006danger,rieck2024geometric} mostly use special Cartesian coordinate systems to eliminate (\ref{tetra_form}) into a univariate equation, then use the discriminant to obtain certain conditions. In contrast, this work adopts the local dual space to deal  with singular solutions of (\ref{tetra_form}) in their multivariate nature, which leads to a systematic 
 stratification framework.

Let $k\in\mathbb{N}$, suppose that $\dim \mathcal{D}_{f,\xi}\geq k$ and the Jacobian $J_f(\xi)$ is of corank one, then a computational criterion can be found in \cite{li2022improved}: if a deterministic $\Delta_{k}\in\mathbb{C}[X]^{*}$ of degree $k$ satisfies $u^H\Delta_{k}[f](\xi)=0$, where $u\in \mathbb{C}^n$ is the null vector of $J_f(\xi)^H$, then $\dim \mathcal{D}_{f,\xi}\geq k+1$ (see \cite[Theorem 1]{li2022improved}). In fact, for the P3P system (\ref{tetra_form}), $\forall\xi\in\mathcal{V}_{\geq \mu}$ satisfies $\mathrm{corank}~J_f(\xi)=1$ (see Lemma~\ref{breadthone}), so let $\mathcal{J}_k=\langle u^H\Delta_k[f] \rangle$, then
\begin{equation}
    \mathcal{V}_{\geq \mu} \subseteq\mathbb{V}\left(\sum_{k=1}^{\mu-1}\mathcal{J}_k\right), ~~\mu=2,3,4,
    \nonumber
\end{equation}
where the difference between two sets is due to some underlying degeneracy. Therefore, the geometric stratification (\ref{stra}) can be obtained by decomposing $\sqrt{\sum \mathcal{J}_k}$ into irreducible components and discarding degenerate ones.
The derivations in the next section are based on the framework from Figure~\ref{research framework}.
\begin{figure}[H]
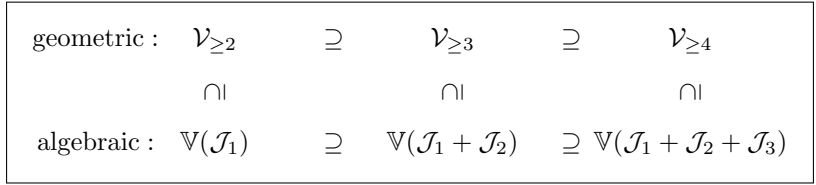

\centering
\[
\boxed{
    \begin{codi}
    \obj[row sep=0.4em]{
     \mathrm{geometric:} & \mathcal{V}_{\geq2} & \supseteq & \mathcal{V}_{\geq3} & \supseteq&\mathcal{V}_{\geq4} \\
     & \rotatebox{90}{$\supseteq$} & & \rotatebox{90}{$\supseteq$} & & \rotatebox{90}{$\supseteq$}\\
    \mathrm{algebraic:} & \V(\J_1)&\supseteq & \V(\J_1+\J_2) &\supseteq & \V(\J_1+\J_2+\J_3) \\
    };
    \end{codi}
}
\]
\caption{Stratification framework I.}
\label{research framework}
\end{figure}

\begin{rmk} 
Note that in a concrete computational setting, the input consists of polynomials in variables—such as $e_1, e_2, e_3$—with rational coefficients, and all algebraic computations are performed within polynomial rings over $\mathbb{Q}$.
\end{rmk}

\section{Geometric Stratification for Singular Configurations}\label{geometric stratification}

Given a 3D triangle $\triangle ABC$, i.e., $(s_{12}, s_{13}, s_{23})$ is fixed, the Jacobian matrix of (\ref{tetra_form}) is computed as

\begin{equation}
    \mathbf{J}=\begin{bmatrix}
  0& \frac{s_{23}^2+e_2^2-e_3^2}{e_2}  & \frac{s_{23}^2-e_2^2+e_3^2}{e_3} \vspace{5pt} \\
  \frac{s_{13}^2+e_1^2-e_3^2}{e_1}& 0 & \frac{s_{13}^2-e_1^2+e_3^2}{e_3} \vspace{5pt} \\
  \frac{s_{12}^2+e_1^2-e_2^2}{e_1}&\frac{s_{12}^2-e_1^2+e_2^2}{e_2}  &0
\end{bmatrix},
\nonumber
\end{equation}
where $\cos\theta_{ij}$ are eliminated by the Law of Cosines.

If $(e_1, e_2, e_3)$ is a singular solution of (\ref{tetra_form}), then it is easy to get $\mathrm{rank}~\mathbf{J}\leq 2$. We show in the following lemma that $\mathrm{corank}~\mathbf{J}=1$ exactly.
\begin{lem}\label{breadthone}
Suppose $\triangle ABC$ is a valid triangle, then  $\mathrm{rank}~\mathbf{J}\geq 2$.
\end{lem}
\begin{proof}
    Observe that each row of $\mathbf{J}$ has two off-diagonal entries with numerators
    \begin{equation}
        s_{ij}^2+e_i^2-e_j^2~~\text{and}~~s_{ij}^2-e_i^2+e_j^2.
        \nonumber
    \end{equation}
    Clearly, at least one of them is nonzero as long as $s_{ij}\neq0$.

    Without loss of generality, assume $s_{23}^2+e_2^2-e_3^2\neq0$, then at least one of
    \begin{equation}
        \left|\begin{array}{cc}
  0& \frac{s_{23}^2+e_2^2-e_3^2}{e_2}  \vspace{5pt} \\
  \frac{s_{13}^2+e_1^2-e_3^2}{e_1}& 0
\end{array}\right| ~~\text{and}~~
\left|\begin{array}{cc}
  \frac{s_{23}^2+e_2^2-e_3^2}{e_2}  & \frac{s_{23}^2-e_2^2+e_3^2}{e_3} \vspace{5pt} \\
   0 & \frac{s_{13}^2-e_1^2+e_3^2}{e_3},
\end{array}\right|
\nonumber
    \end{equation}
    is nonzero, which concludes the proof.
\end{proof}

Assuming that $\mathbf{J}_{23},\mathbf{J}_{32}$ are not zero, the left null vector of $\mathbf{J}$ is computed as
\begin{equation}
    u=\begin{bmatrix}
 1 \vspace{5pt}\\
 \frac{-e_2^2 + e_3^2 + s_{23}^2}{e_1^2 - e_3^2 - s_{13}^2} \vspace{5pt} \\
\frac{e_2^2 - e_3^2 + s_{23}^2}{e_1^2 - e_2^2 - s_{12}^2}
\end{bmatrix}.
\nonumber
\end{equation}

Now, we are ready to proceed with the 
 framework in Figure~\ref{research framework}. Note that all computation results are illustrated in the attached Maple worksheet.

\subsection{\texorpdfstring{$\mathcal{V}_{\geq2}\subseteq\V(\J_1)$}{}}
\subsubsection{\texorpdfstring{Decomposing $\V(\J_1)$}{}}
Suppose $\dim \mathcal{D}_{f,\xi}\geq 1$, then the first order criterion is computed as
\begin{equation}
    \begin{aligned}
     u^H\Delta_1[f]=&2e_2e_3( e_{1}^{4} s_{23}^{2}+e_{1}^{2} e_{2}^{2} s_{12}^{2}-e_{1}^{2} e_{2}^{2} s_{13}^{2}-e_{1}^{2} e_{2}^{2} s_{23}^{2}-e_{1}^{2} e_{3}^{2} s_{12}^{2} \\
     & +e_{1}^{2} e_{3}^{2} s_{13}^{2}-e_{1}^{2} e_{3}^{2} s_{23}^{2}+e_{2}^{4} s_{13}^{2}-e_{2}^{2} e_{3}^{2} s_{12}^{2}-e_{2}^{2} e_{3}^{2} s_{13}^{2} \\
     & +e_{2}^{2} e_{3}^{2} s_{23}^{2}+e_{3}^{4} s_{12}^{2}-s_{12}^{2} s_{13}^{2} s_{23}^{2} ).
\end{aligned}
\nonumber
\end{equation}
By omitting $2e_2e_3$, we derive
\begin{equation}
    \begin{aligned}
     \J_1=&\left \langle u^H\Delta_1[f]/2e_2e_3 \right \rangle \\
     =&\left \langle e_{1}^{4} s_{23}^{2}+e_{1}^{2} e_{2}^{2} s_{12}^{2}-e_{1}^{2} e_{2}^{2} s_{13}^{2}-e_{1}^{2} e_{2}^{2} s_{23}^{2}-e_{1}^{2} e_{3}^{2} s_{12}^{2} \right.\\
     & \left.+e_{1}^{2} e_{3}^{2} s_{13}^{2}-e_{1}^{2} e_{3}^{2} s_{23}^{2}+e_{2}^{4} s_{13}^{2}-e_{2}^{2} e_{3}^{2} s_{12}^{2}-e_{2}^{2} e_{3}^{2} s_{13}^{2}\right. \\
     & \left.+e_{2}^{2} e_{3}^{2} s_{23}^{2}+e_{3}^{4} s_{12}^{2}-s_{12}^{2} s_{13}^{2} s_{23}^{2} \right \rangle,
\end{aligned}
\nonumber
\end{equation}
which is a prime ideal.

Consequently, $\mathcal{V}_{\geq2}=\V(\J_1)$.

\subsubsection{\texorpdfstring{Interpreting $\mathcal{V}_{\geq2}$}{}}
To clarify the geometric meaning of $\V(\J_1)$, consider placing $\triangle ABC$ in the following coordinate system
\begin{equation}\label{coordinate}
    A=(0,\ 0,\ 0), \quad B=(x_2,\ 0,\ 0), \quad C=(x_3,\ y_3,\ 0),
\end{equation}
where $x_2=s_{12},x_3=s_{13}\cos{\angle A},y_3=s_{13}\sin{\angle A}$. Let $P=(x,y,0)$ be the projection of the camera center $O=(x,y,z)$ onto the $\triangle ABC$ plane (see Figure~\ref{mu2}(b)), then $e_1,e_2,e_3$ can be computed as
\begin{equation}
    \begin{aligned} \label{edge}
    &e_1=\sqrt{x^2+y^2+z^2},\\
    &e_2=\sqrt{(x-x_2)^2+y^2+z^2},\\
    &e_3=\sqrt{(x-x_3)^2+(y-y_3)^2+z^2}.
\end{aligned}
\end{equation}

Solving $x,y$ from (\ref{edge}) then substituting them into the equation of the circumcircle of $\triangle ABC$ as
\begin{equation}
    x^2+y^2-x_2x+\frac{x_2x_3-x_3^2-y_3^2}{y_3}y=0,
    \nonumber
\end{equation}
we derive
\begin{equation}\label{cylinder}
    \begin{aligned}
    & e_{1}^ {4} s_{23}^{2}+e_{1}^{2} e_{2}^{2} s_{12}^{2}-e_{1}^{2} e_{2}^{2} s_{13}^{2}-e_{1}^{2} e_{2}^{2} s_{23}^{2}-e_{1}^{2} e_{3}^{2} s_{12}^{2}\\
     & +e_{1}^{2} e_{3}^{2} s_{13}^{2}-e_{1}^{2} e_{3}^{2} s_{23}^{2}+e_{2}^{4} s_{13}^{2}-e_{2}^{2} e_{3}^{2} s_{12}^{2}-e_{2}^{2} e_{3}^{2} s_{13}^{2}\\
     & +e_{2}^{2} e_{3}^{2} s_{23}^{2}+e_{3}^{4} s_{12}^{2}-s_{12}^{2} s_{13}^{2} s_{23}^{2}=0,
\end{aligned}
\end{equation}
where a factor of $(s_{23}-s_{13}+s_{12})(s_{23}-s_{13}-s_{12})(s_{23}+s_{13}+s_{12})(s_{23}+s_{13}-s_{12})$ is omitted. Clearly $\J_1$ is the same as (\ref{cylinder}), which defines a right cylinder with the circumcircle of $\triangle ABC$ as its base, i.e., the ``danger cylinder'' (see Figure~\ref{mu2}(a)).

\begin{thm}\label{thm1}
     $(e_1,e_2,e_3)$ is an isolated solution of (\ref{tetra_form}) with multiplicity $\mu \ge 2$, {\emph{if and only if}} the camera center $O$ lies on the danger cylinder defined by $\triangle ABC$.
 \end{thm}

 \begin{figure}[H]
\centering  
\subfigure[]{
\includegraphics[height= 0.33\linewidth]{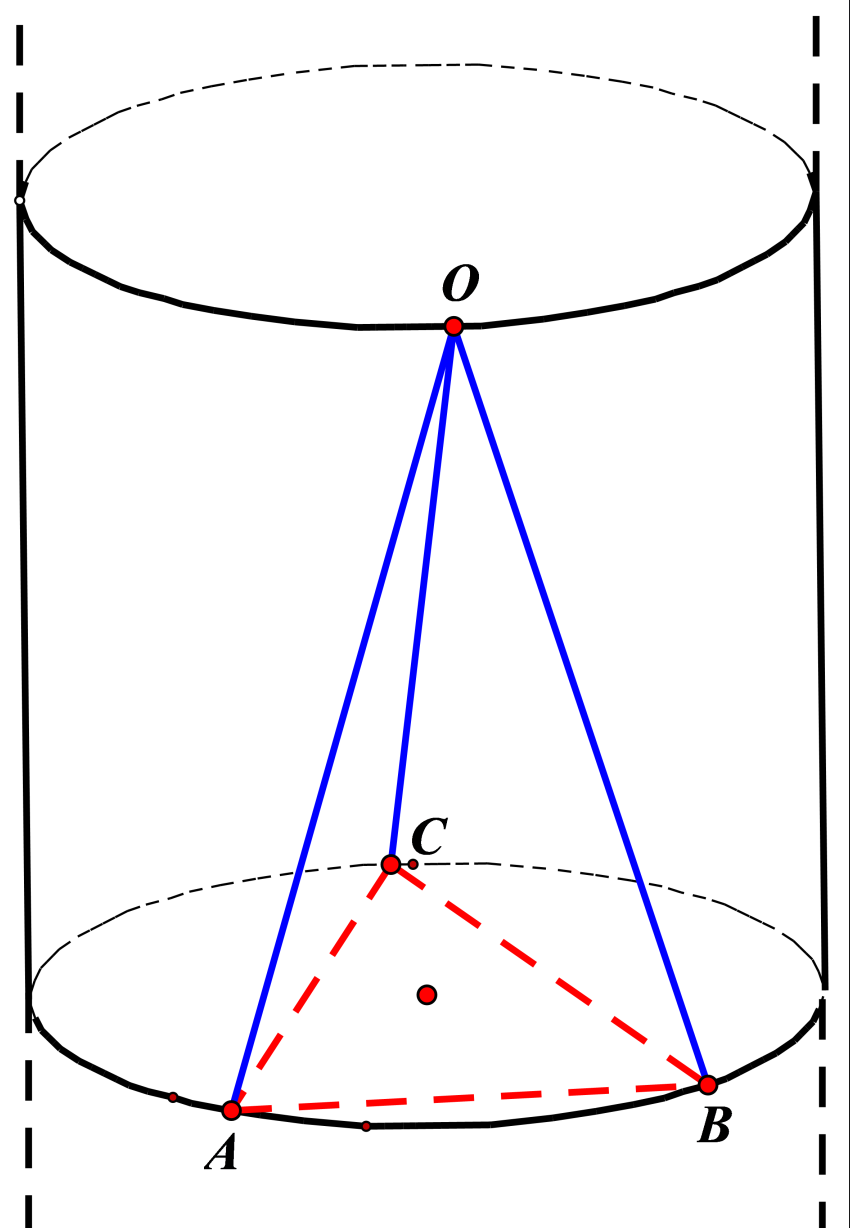}}
\subfigure[]{
\includegraphics[height= 0.33\linewidth]{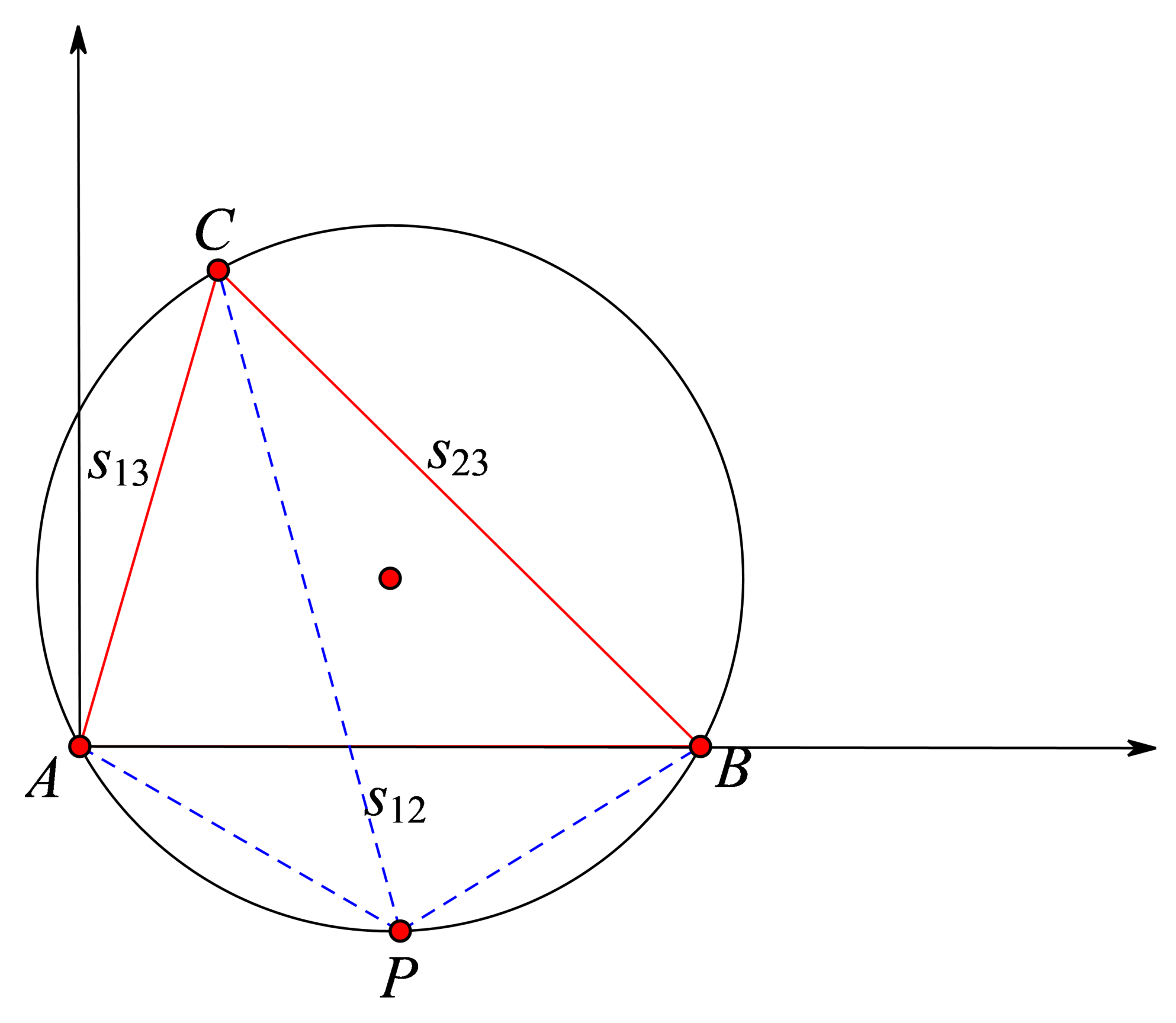}}
\caption{The geometry of $\mathcal{V}_{\geq2}$.}
\label{mu2}
\end{figure}

\subsection{\texorpdfstring{$\mathcal{V}_{\geq3}\subseteq\V(\J_1+\J_2)$}{}}
\subsubsection{\texorpdfstring{Decomposing $\V(\J_1+\J_2)$}{}}
\label{sec:I12}
Suppose $\dim \mathcal{D}_{f,\xi}\geq 2$, then $\J_2$ is computed as
\begin{equation}
    \begin{aligned}
        \J_2=&\left \langle u^H\Delta_2[f]/2e_2^2e_3^2\right \rangle \\
        =&\left \langle e_1^{12}s_{23}^2+e_1^{10}e_2^2s_{12}^2-e_1^{10}e_2^2s_{13}^2-3e_1^{10}e_2^2s_{23}^2-e_1^{10}e_3^2s_{12}^2 \right.\\
     &\left.+e_1^{10}e_3^2s_{13}^2+\cdots\text{219 terms}\cdots+e_3^8s_{12}^6 -2e_3^6s_{12}^6s_{13}^2\right. \\
     &\left.-7e_3^4s_{12}^6s_{13}^4+3e_3^4s_{12}^6s_{13}^2s_{23}^2+2e_3^2s_{12}^6s_{13}^4s_{23}^2+3s_{12}^6s_{13}^6s_{23}^2  \right \rangle.
    \end{aligned}
    \nonumber
\end{equation}

Since directly decomposing $\sqrt{\J_1+\J_2}$ is computationally expensive, we
 employ a specialized algebraic approach
 to derive the result. We provide the decomposition result below, whose rigorous proof can be found in~\ref{app I1}.

Introduce four ideals as follows.
\begin{itemize}
\item $\I_1$: geometrically corresponds to three special generatrices of the danger cylinder (which will be explained later) as
\begin{equation}
        \I_1=\J_1 + \mathcal{G},
        \nonumber
    \end{equation}
    where
   \begin{equation}
    \begin{aligned}
    \mathcal{G}=
    &\left \langle -e_1^{4} s_{12}^{2}+e_1^{4} s_{13}^{2}-2 e_1^{2} e_2^{2} s_{13}^{2}+2 e_1^{2} e_3^{2} s_{12}^{2}+e_2^{4} s_{13}^{2}-e_2^{2} s_{12}^{2} s_{13}^{2}-e_3^{4} s_{12}^{2}+e_3^{2} s_{12}^{2} s_{13}^{2},\right.\\
    &\left. e_1^{4} s_{23}^{2}-2 e_1^{2} e_2^{2} s_{23}^{2}-e_1^{2} s_{12}^{2} s_{23}^{2}-e_2^{4} s_{12}^{2}+e_2^{4} s_{23}^{2}+2 e_2^{2} s_{12}^{2} e_3^{2}-e_3^{4} s_{12}^{2}+e_3^{2} s_{12}^{2} s_{23}^{2},\right.\\
    &\left. e_1^{4} s_{23}^{2}-2 e_1^{2} e_3^{2} s_{23}^{2}-s_{13}^{2} e_1^{2} s_{23}^{2}-e_2^{4} s_{13}^{2}+2 e_2^{2} e_3^{2} s_{13}^{2}+e_2^{2} s_{13}^{2} s_{23}^{2}-e_3^{4} s_{13}^{2}+e_3^{4} s_{23}^{2}\right \rangle.
\end{aligned}
\nonumber
\end{equation}
    \item $\I_2$: geometrically corresponds to the circumcircle of $\triangle ABC$ as
    \begin{equation}
        \I_2=\J_1 + \mathcal{H},
        \nonumber
    \end{equation}
    where
    \begin{equation}
    \begin{aligned}
     \mathcal{H}=&\left \langle-e_1^{4}s_{23}^{2}-e_1^{2}e_2^{2}s_{12}^{2}+e_1^{2}e_2^{2}s_{13}^{2}+e_1^{2}e_2^{2}s_{23}^{2}+e_1^{2} e_3^{2} s_{12}^{2}-e_{1}^{2} e_{3}^{2} s_{13}^{2}\right.\\
    &\left.+e_{1}^{2}e_{3}^{2} s_{23}^{2}+e_{1}^{2}s_{12}^{2} s_{23}^{2}+e_{1}^{2} s_{13}^{2} s_{23}^{2}-e_{1}^{2}s_{23}^{4}-e_{2}^{4} s_{13}^{2}+e_{2}^{2} e_{3}^{2} s_{12}^{2}\right.\\
    &\left.+e_{2}^{2} e_{3}^{2} s_{13}^{2}-e_{2}^{2} e_{3}^{2} s_{23}^{2}+e_{2}^{2} s_{12}^{2} s_{13}^{2}-e_{2}^{2} s_{13}^{4}+e_{2}^{2} s_{13}^{2} s_{23}^{2}-e_{3}^{4} s_{12}^{2} \right.\\
    &\left.-e_{3}^{2}s_{12}^{4}+e_{3}^{2} s_{12}^{2} s_{13}^{2}+e_{3}^{2} s_{12}^{2}s_{23}^{2}-s_{12}^{2} s_{13}^{2}s_{23}^{2}\right \rangle,
    \end{aligned}
    \nonumber
    \end{equation}
    which geometrically corresponds to the square of the volume of the tetrahedron $O\triangle ABC$.
    \item $\I_3$: geometrically corresponds to the union of points on the danger cylinder whose projection coincides with $B$ or $C$ or the antipodal point of $A$ as
    \begin{equation}
        \I_3=\I_{31}\cap\I_{32}\cap\I_{33},
        \nonumber
    \end{equation}
    where
    \begin{equation}
        \begin{aligned}
            \I_{31}&=\J_1+\left \langle e_{2}^{2}-e_{3}^{2}+s_{23}^{2},-e_{1}^{2}+e_{2}^{2}+s_{12}^{2}   \right \rangle, \\
            \I_{32}&=\J_1+\left \langle -e_{2}^{2}+e_{3}^{2}+s_{23}^{2},-e_{1}^{2}+e_{3}^{2}+s_{13}^{2}  \right \rangle, \\
            \I_{33}&=\J_1+\left \langle -e_{1}^{2}+e_{3}^{2}+s_{13}^{2},-e_{1}^{2}+e_{2}^{2}+s_{12}^{2}  \right \rangle.
        \end{aligned}
        \nonumber
    \end{equation}
    Note that three ideals on the right correspond to the degenerate cases in which the nonzero assumption of $\mathbf{J}_{23},\mathbf{J}_{32}$ fails.
    \item $\I_4$: corresponds to some trivial degenerate cases as
    \begin{equation}
    \begin{aligned}
    \I_4=
    &\left \langle s_{12},-e_1+e_2\right \rangle  \cap \left \langle s_{12},e_1+e_2 \right \rangle \cap \left \langle s_{13}, -e_1+e_3 \right \rangle \cap \left \langle s_{13},e_1+e_3 \right \rangle \cap \left \langle s_{12}, s_{13}, s_{23} \right \rangle.
    \end{aligned}
    \nonumber
\end{equation}
\end{itemize}

The following lemma is frequently used to prove the subset relation between two varieties.

\begin{lem}[Proposition 12 of \S 4.4 \cite{cox1997ideals}]\label{Ideal contain}
    Let $\I$ and $\J$ be ideals in $k[x_1,\dots,x_n]$.
    \begin{itemize}
        \item (i) $\J\subseteq\I$ if and only if $\I:\J=k[x_1,\dots,x_n]$.
        \item (ii) $\J\subseteq \sqrt{\I}$ if and only if $\I:\J^{\infty}=k[x_1,\dots,x_n]$.
    \end{itemize}
\end{lem}

\begin{thm}\label{J1+J2}
    Suppose $\triangle ABC$ is a valid triangle, then
\begin{equation*}
\V(\J_1+\J_2)=\V(\I_1)\cup\V(\I_2)\cup\V(\I_3)\cup\V(\I_4).
\end{equation*}
\end{thm}
\begin{proof}
    Let $\mathcal{K}=\mathcal{I}_1\cap\mathcal{I}_2\cap\mathcal{I}_3\cap\mathcal{I}_4$, then the following computation
    \begin{equation}
        1\in \sqrt{\J_1+\J_2}:\K\ \text{and}\ S_{\triangle ABC} \in \K:\sqrt{\J_1+\J_2},
        \nonumber
    \end{equation}
    imply $\sqrt{\J_1+\J_2}=\K$ as long as the area $S_{\triangle ABC}\neq 0$.
\end{proof}
Since $\V(\I_3),\V(\I_4)$ correspond to some degenerate cases (pseudo triple-zeros or invalid triangles), we conclude $\mathcal{V}_{\geq3}=\V(\I_1)\cup\V(\I_2)$.

\subsubsection{\texorpdfstring{Interpreting $\mathcal{V}_{\geq3}$}{}}
\label{sec:vI1}
We establish the geometric meaning of $\V(\I_1)$.

\begin{defi}[\cite{kimberling1998triangle}]
    The first Morley triangle $\triangle DEF$, short for ``Morley's triangle", is the triangle constructed from pairwise intersections of the angle trisectors of a given triangle $\triangle ABC$ (see Figure~\ref{Morley triangle}(a)).
\end{defi}

\begin{figure}[H]
\centering  
\subfigure[]{
\includegraphics[height= 0.33\linewidth]{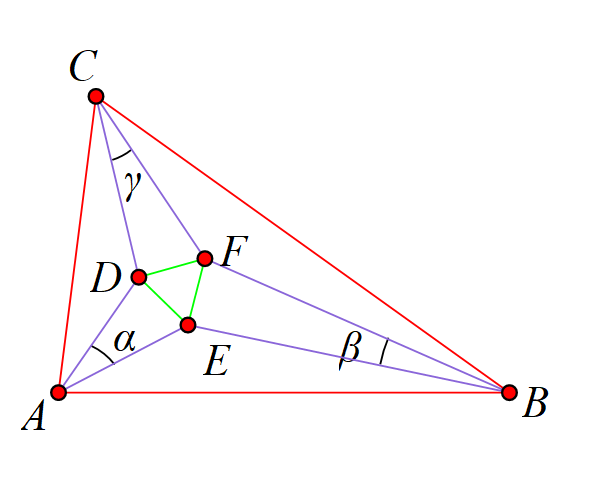}}
\subfigure[]{
\includegraphics[height= 0.33\linewidth]{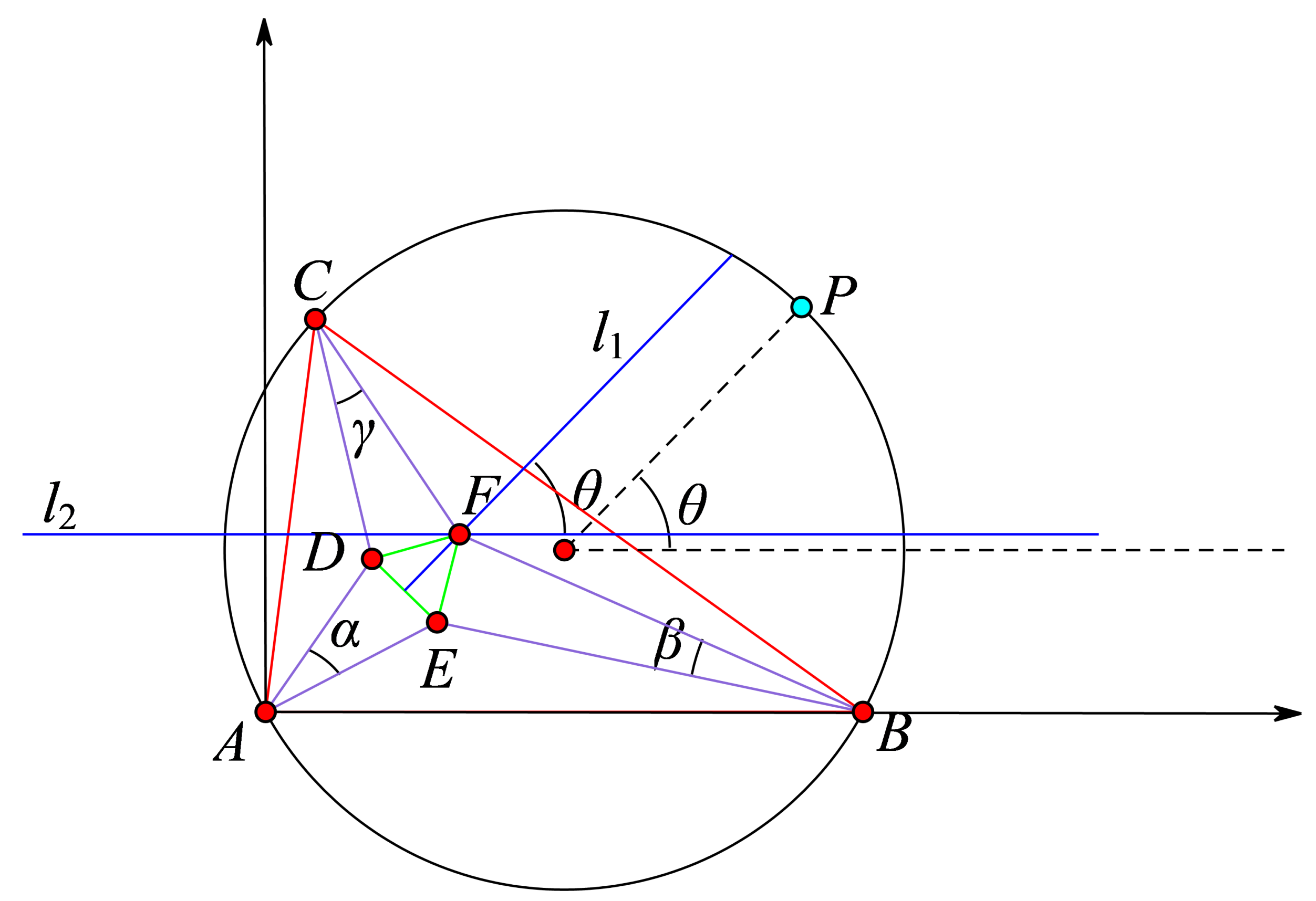}}
\caption{The first Morley triangle.}
    \label{Morley triangle}
\end{figure}

\begin{thm}
    \textbf{(Morley's theorem)} Morley triangle is equilateral.
\end{thm}

We continue to use the coordinate system (\ref{coordinate}). Suppose $\triangle DEF$ is the Morley triangle of $\triangle ABC$ (see Figure~\ref{Morley triangle}(b)), let $\theta$ denote the angle between $l_1$ and $l_2$, namely the altitude on $DE$ and the direction of $AB$, then
the camera center $O$ and its projection $P$ can be parameterized as
\begin{equation}
    \begin{aligned}
       O=\left(\frac{x_2}{2}+R\cos{\theta},\ \frac{x_3^2+y_3^2-x_2x_{3}}{2y_3}+R\sin{\theta},\ h\right),\\
       P=\left(\frac{x_2}{2}+R\cos{\theta},\ \frac{x_3^2+y_3^2-x_2x_{3}}{2y_3}+R\sin{\theta},\ 0\right),
    \end{aligned}
    \nonumber
\end{equation}
where $R$ is the radius of the circumcircle of $\triangle ABC$ satisfying
\begin{equation}
    R=\sqrt{\left(\frac{x_2}{2}\right)^2+\left(\frac{x_3^2+y_3^2-x_2x_3}{2y_3}\right)^2},
    \nonumber
\end{equation}
$h$ is the height of $O$ over the $\triangle ABC$ plane, and $\sin{\theta}$ satisfies $f_1=0$, where
\begin{equation}
    \begin{aligned}
        f_1:=&8s_{13}s_{23}s_{12}^{2}\sin^{3}\theta-6s_{13}s_{23}s_{12}^{2}\sin\theta+s_{12}^{2}s_{13}^{2}+s_{12}^{2}s_{23}^{2}-s_{13}^{4}+2s_{13}^{2}s_{23}^{2}-s_{23}^{4},
    \end{aligned}
    \nonumber
\end{equation}
which defines the location of $P$ with respect to $\triangle DEF$. The proofs of Morley's theorem and $f_1$ can be found in \ref{app Morley}.

Applying the Pythagorean theorem to triangles $\triangle OPA,\triangle OPB$, and $\triangle OPC$ respectively yields the following three distance equations $f_2=f_3=f_4=0$, where
\begin{equation}
    \begin{aligned}
    &f_2:=(x + R \cos \theta)^2 + (y + R \sin \theta)^2 + h^2 - e_1^2, \\
    &f_3:=(x + R \cos \theta - x_2)^2 + (y + R \sin \theta)^2 + h^2 - e_2^2,  \\
    &f_4:=(x + R \cos \theta - x_3)^2 + (y + R \sin \theta - y_3)^2 + h^2 - e_3^2, \nonumber
\end{aligned}
\end{equation}
and $(x,y)$ is the circumcenter of $\triangle OPC$.
Finally, we introduce the trigonometric identity $f_5=0$, where
\begin{equation}
    f_5:=\sin^2 \theta+\cos^2 \theta-1.
    \nonumber
\end{equation}

\begin{defi}[Definition 1 of \S 3.1 \cite{cox1997ideals}]\label{Elimination Ideal}
    Given $\I  
    \subseteq k[x_1,\dots,x_l,x_{l+1},\ldots,x_n]$, the $l$-th elimination ideal $\I_l$ is the ideal of $k[x_{l+1},\dots,x_n]$ defined by
    \begin{equation}
        \I_l=\I\cap k[x_{l+1},\dots,x_n].
        \nonumber
    \end{equation}
\end{defi}

Let $\F=\left \langle f_1,f_2,f_3,f_4,f_5 \right \rangle \subseteq k[\sin \theta, \cos \theta,h,e_1,e_2,e_3,s_{12},s_{13},s_{23}]$, then
\begin{equation}
    \begin{aligned}
    \IM=&\F\cap k[e_1,e_2,e_3,s_{12},s_{13},s_{23}]\\
    =&\J_1+\left \langle  -e_1^{6}s_{12}^{6}+3e_1^{6}s_{12}^{4}s_{13}^{2}+e_1^{6}s_{12}^{4}s_{23}^{2}-3e_1^{6}s_{12}^{2}s_{13}^{4}+2e_1^{6}s_{12}^{2}s_{13}^{2}s_{23}^{2}+s_{12}^{6}s_{13}^{4}s_{23}^{2}\right.  \\
    &\left.-s_{12}^{6}s_{13}^{2}s_{23}^{4}+s_{12}^{4}s_{13}^{6}s_{23}^{2}+\cdots \text{60 terms}\cdots-2s_{12}^{4}s_{13}^{4}s_{23}^{4}+s_{12}^{4}s_{13}^{2}s_{23}^{6}, \right.  \\
    &\left. e_1^{8}s_{12}^{8}-4e_1^{8}s_{12}^{6}s_{13}^{2}-4e_1^{8}s_{12}^{6}s_{23}^{2}+6e_1^{8}s_{12}^{4}s_{13}^{4}+4e_1^{8}s_{12}^{4}s_{13}^{2}s_{23}^{2}+2e_3^{2}s_{12}^{6}s_{13}^{6}s_{23}^{2}\right.\\
    &\left.-4e_3^{2}s_{12}^{6}s_{13}^{4}s_{23}^{4}+\cdots \text{106 terms} \cdots+2e_3^{2}s_{12}^{6}s_{13}^{2}s_{23}^{6} +4s_{12}^{8}s_{13}^{4}s_{23}^{4}\right \rangle
\end{aligned}
\nonumber
\end{equation}

The geometric meaning of $\V(\IM)$ can be described as follows: it consists of three generatrices of the danger cylinder, where each generatrix intersects the circumcircle at the intersection of the circumcircle and a line passing through the circumcenter of $\triangle ABC$ and parallel to one altitude of $\triangle DEF$, as shown in Figure~\ref{fig:mu3}(a,b).

By computational verification using Lemma~\ref{Ideal contain}, we derive
\begin{equation}
    1\in \IM:\I_1^\infty \quad \text{and}\quad 1\in \I_1:\IM^\infty,
    \nonumber
\end{equation}
which implies $\V(\IM)=\V(\I_1)$. 

\begin{thm}
    $(e_1,e_2,e_3)$ is a singular solution of (\ref{tetra_form}) with multiplicity $\mu \ge 3$, \emph{if and only if} the camera center $O$ lies on one of three generatrices of the danger cylinder associated with the first Morley triangle or the circumcircle of $\triangle ABC$ (see Figure~\ref{fig:mu3}).
\end{thm}

\begin{figure}[H]
\centering  
\subfigure[]{
\includegraphics[height=0.33\textwidth]{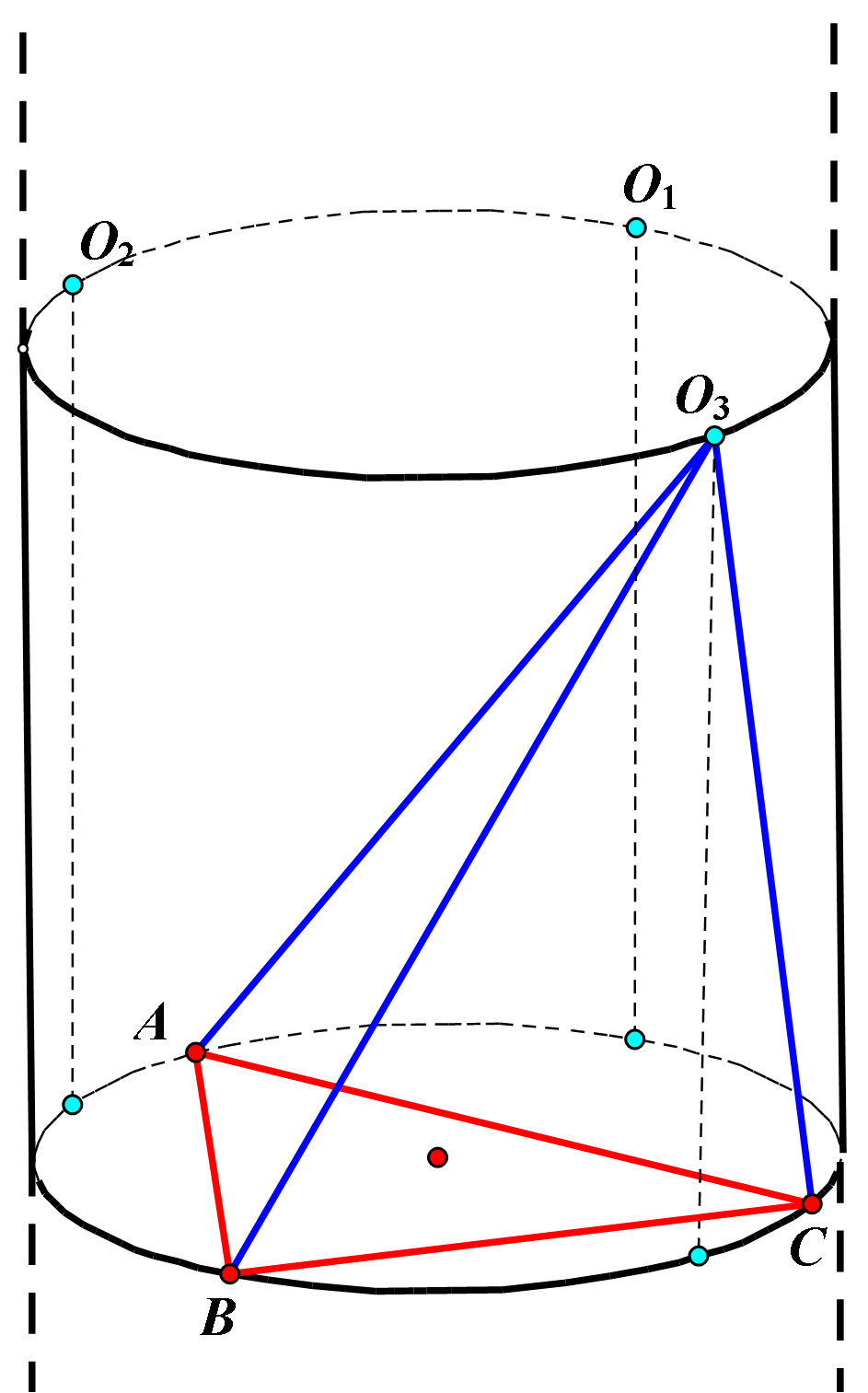}}
\subfigure[]{
\includegraphics[height=0.33\textwidth]{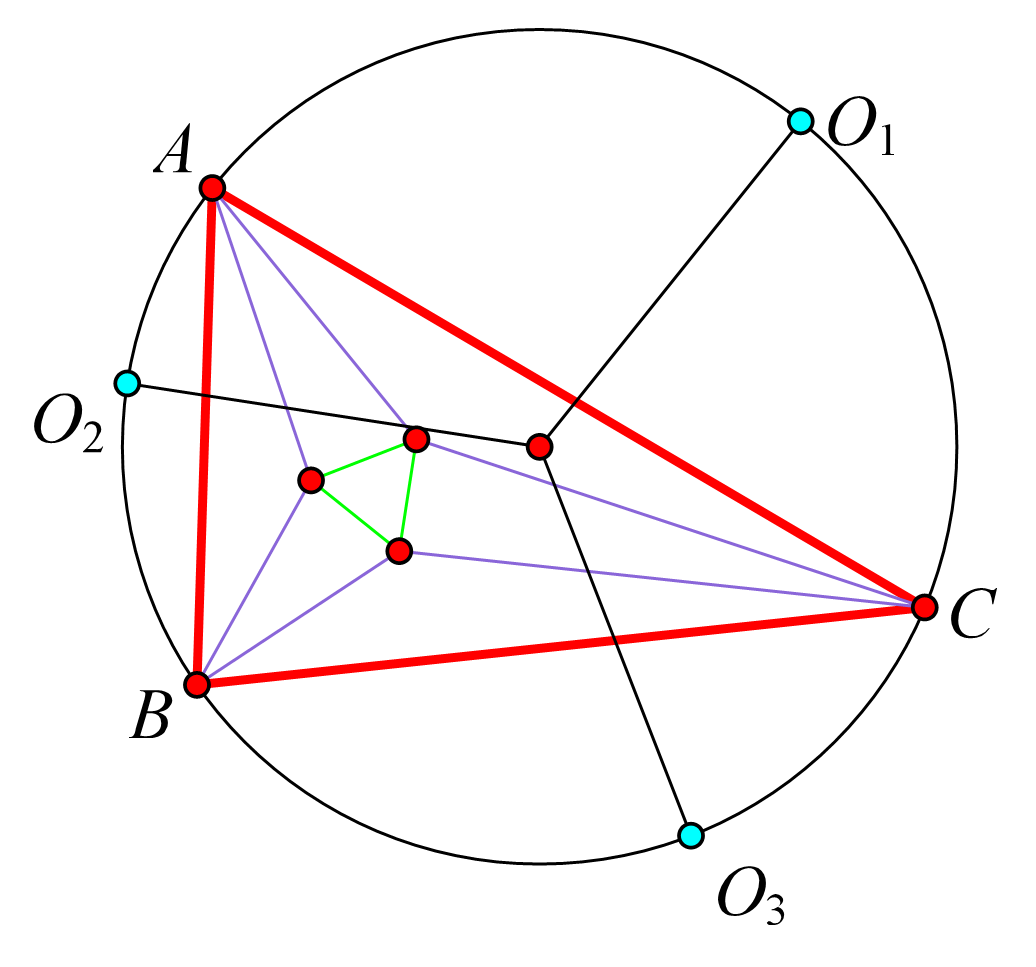}}
\subfigure[]{
\includegraphics[height=0.33\textwidth]{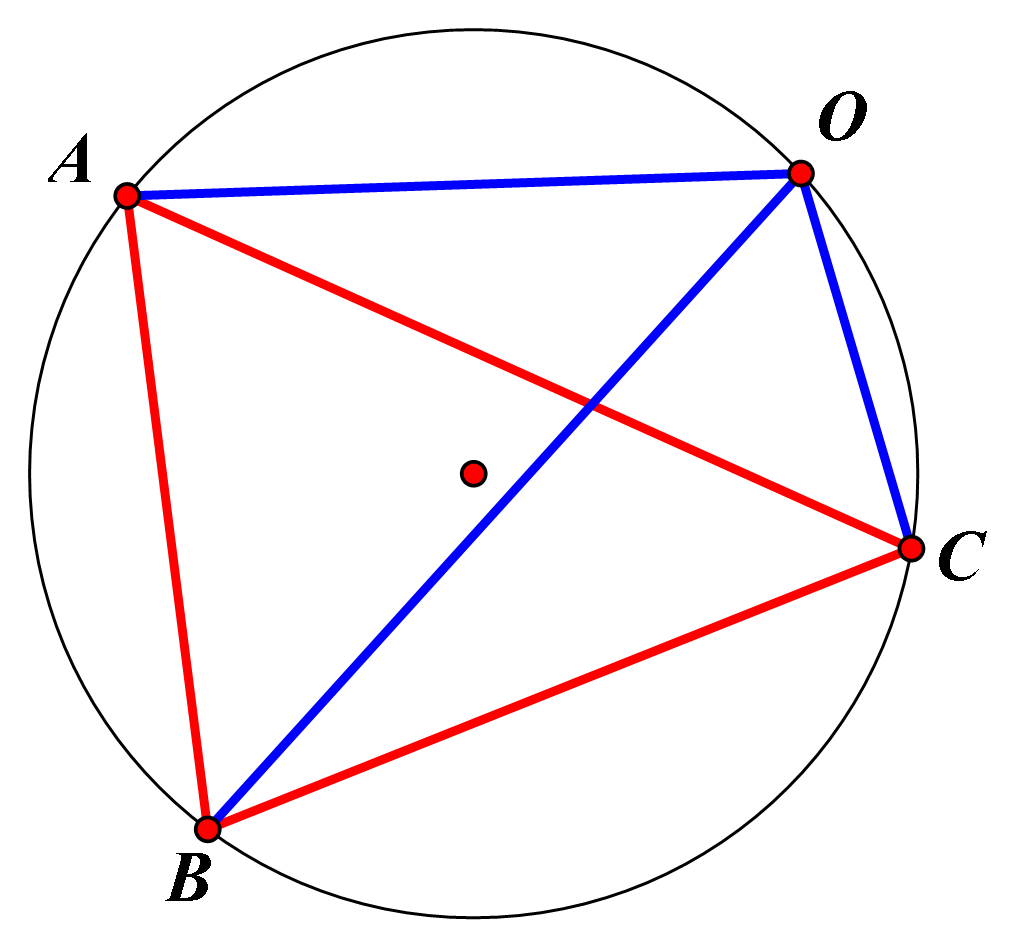}}
\caption{The geometry of $\mathcal{V}_{\geq3}$.}
\label{fig:mu3}
\end{figure}

\subsection{\texorpdfstring{$\mathcal{V}_{\geq4}\subseteq\V(\J_1+\J_2+\J_3)$}{}}
\subsubsection{\texorpdfstring{Decomposing $\V(\J_1+\J_2+\J_3)$}{}}
Suppose $\dim \mathcal{D}_{f,\xi}\geq 3$, then $\J_3$ is computed as
\begin{equation}
    \begin{aligned}
     \J_3=&\left \langle u^H\Delta_3[f]/8e_2^3e_3^3 \right \rangle\\
     =&\left \langle -e_1^{16}e_2^2s_{12}^4+e_1^{16}e_2^2s_{12}^2s_{13}^2+e_1^{16}e_3^2s_{12}^2s_{13}^2-e_1^{16}e_3^2s_{13}^4-e_1^{16}s_{12}^2s_{13}^2s_{23}^2 \right.\\
     &\left.+2e_1^{14}e_2^4s_{12}^4-4e_1^{14}e_2^4s_{12}^2s_{13}^2+\cdots\text{745 terms}\cdots+2e_3^8s_{12}^{10}s_{13}^4-e_3^6s_{12}^{10}s_{13}^4s_{23}^2\right. \\
     & \left.+4e_3^4s_{12}^{10}s_{13}^8-4e_3^4s_{12}^{10}s_{13}^6s_{23}^2-e_3^2s_{12}^{10}s_{13}^{8}s_{23}^2-s_{12}^{10}s_{13}^{10}s_{23}^2 \right \rangle.
\end{aligned}
\nonumber
\end{equation}

Due to the enormous computational cost of calculating the radical ideal of $\J_1+\J_2+\J_3$ and its decomposition, we adopt an alternative approach
\begin{equation}
    \begin{aligned}
    \V(\J_1+\J_2+\J_3)&=\V(\J_1+\J_2)\cap\V(\J_3) \\
    &=(\V(\I_1)\cup \V(\I_2)\cup \V(\I_3)\cup \V(\I_4))\cap \V(\J_3) \\
    &=(\V(\I_1)\cap \V(\J_3))\cup (\V(\I_2)\cap \V(\J_3))\cup \\
    &\ \ \ \ (\V(\I_3)\cap \V(\J_3))\cup (\V(\I_4)\cap \V(\J_3)).
\end{aligned}
\nonumber
\end{equation}

Since $\V(\I_3),\V(\I_4)$ correspond to degenerate cases, the geometrically meaningful components are $\V(\I_1)\cap \V(\J_3),\V(\I_2)\cap \V(\J_3)$.
\begin{lem}
    Suppose $\triangle ABC$ is a valid triangle, then
    \begin{equation}
        \V(\I_1+\J_3)\subseteq\V(\I_2)\subseteq\V(\J_3).
        \nonumber
    \end{equation}
\end{lem}

\begin{proof}
    The following computation
    \begin{equation}
        1\in (\I_1+\J_3):\I_2^{\infty}\ \text{and}\ S_{\triangle ABC}\in \I_2:\J_3,
        \nonumber
    \end{equation}
    imply $\I_2\subseteq \sqrt{\I_1+\J_3}$ and $\J_3\subseteq \I_2$ as long as $S_{\triangle ABC}\neq 0$.
\end{proof}

Straightforwardly, we derive $\V(\I_1)\cap \V(\J_3)\subseteq\V(\I_2)\cap \V(\J_3)=\V(\I_2)$, which concludes $\mathcal{V}_{\geq4}=\V(\I_2)$.
\subsubsection{\texorpdfstring{Interpreting $\mathcal{V}_{\geq4}$}{}}
\begin{thm}
     $(e_1,e_2,e_3)$ is a singular solution of (\ref{tetra_form}) with multiplicity $\mu \ge 4$, \emph{if and only if} the camera center $O$ lies on the circumcircle of $\triangle ABC$.
\end{thm}
However, when the camera center lies on the circumcircle, the P3P problem admits infinitely many solutions by the Inscribed Angle Theorem. 
We conclude that the P3P problem does not admit a quadruple zero. Thus, over the complex field, the P3P problem always admits either multiple or infinitely many solutions. 

\subsection{Answer to Question~\ref{Q1}}

Recall the question: given $\triangle ABC$, suppose that $(e_1,e_2,e_3)$ is a singular solution of (\ref{tetra_form}) with multiplicity $\geq\mu$,
then where is the location of the corresponding camera center $O$? We summarize the answer as follows.

\begin{itemize}
    \item For $O\in\mathcal{V}_{\geq 2}$, $O$ lies on \textbf{the danger cylinder} of $\triangle ABC$,
    \item For $O\in\mathcal{V}_{\geq 3}$, $O$ lies on one of \textbf{three Morley generatrices} of the danger cylinder or \textbf{the circumcircle} of $\triangle ABC$,
    \item For $O\in\mathcal{V}_{\geq 4}$, $O$ lies on \textbf{the circumcircle} of $\triangle ABC$, whose multiplicity is indeed infinite.
\end{itemize}

\section{Complementary Stratification for Singular Configurations}\label{Geometric Discriminant}
Given a 3D triangle $\triangle ABC$, i.e., $(s_{12}, s_{13}, s_{23})$ is fixed, suppose $(e_1,e_2,e_3)$ is a singular solution of (\ref{tetra_form}) with multiplicity $\geq\mu$ $(\mu=2,3)$, let $(e^\prime_1,e^\prime_2,e^\prime_3)$ denote a complementary solution of (\ref{tetra_form}) associated with $(e_1,e_2,e_3)$, then
\begin{equation}
    \frac{e_i^2+e_j^2-s_{ij}^2}{e_ie_j}=\frac{{e_i^{\prime}}^2+{e_j^{\prime}}^2-{s_{ij}^{\prime}}^2}{e_i^{\prime}e_j^{\prime}},~~i,j\in\{1,2,3\}~\&~i<j,
    \nonumber
\end{equation}
which yields three equations
\begin{equation}
    \begin{aligned}
     &f_1:=-e_2^2e_2^{\prime}e_3^{\prime}+e_2e_3{e_2^{\prime}}^2+e_2e_3{e_3^{\prime}}^2-e_2e_3s_{23}^2-e_3^2e_2^{\prime}e_3^{\prime}+e_2^{\prime}e_3^{\prime}s_{23}^2,\\
     &f_2:=-e_1^2e_1^{\prime}e_3^{\prime}+e_1e_3{e_1^{\prime}}^2+e_1e_3{e_3^{\prime}}^2-e_1e_3s_{13}^2-e_3^2e_1^{\prime}e_3^{\prime}+e_1^{\prime}e_3^{\prime}s_{13}^2,\\
     &f_3:=-e_1^2e_1^{\prime}e_2^{\prime}+e_1e_2{e_1^{\prime}}^2+e_1e_2{e_2^{\prime}}^2-e_1e_2s_{12}^2-e_2^2e_1^{\prime}e_2^{\prime}+e_1^{\prime}e_2^{\prime}s_{12}^2.\\
\end{aligned}
\nonumber
\end{equation}

\begin{lem}[Lemma 1 of \S 3.2 \cite{cox1997ideals}]\label{projection map}
    Consider the projection map
    \begin{equation}
        \pi_l:k^n\longrightarrow k^{n-l}.
        \nonumber
    \end{equation}
    Let $\I \subseteq k[x_1,\ldots,x_n]$ and $\I_l=\I \cap k[x_{l+1},\dots,x_n] $ be the $l$-th elimination ideal, then
    \begin{equation}
        \V(\I_l)=\overline{\pi_l(\V(\I))},
        \nonumber
    \end{equation}
    where $\overline{\bullet} $ denotes the Zariski closure.
\end{lem}

Let $\F=\left \langle f_1,f_2,f_3 \right \rangle \subseteq k[e_1,e_2,e_3,e_1^\prime,e_2^\prime,e_3^\prime,s_{12},s_{13},s_{23}]$, suppose $(e_1,e_2,e_3)\in\V(\K_{\mu})$ is a singular solution of (\ref{tetra_form}) with multiplicity $\geq\mu$, then we intend to compute
\begin{equation}
    \K_{\mu}^\prime=(\F+\K_{\mu})\cap k[e_1^\prime,e_2^\prime,e_3^\prime,s_{12},s_{13},s_{23}],
    \nonumber
\end{equation}
where $\V(\K_{\mu}^\prime)=\overline{\pi(\V(\F+\K_{\mu}))}$ characterizes  the Zariski closure of all complementary solutions.
\begin{figure}[H]
\centering
\[
\boxed{
    \begin{codi}
    \obj[row sep=0.4em]{
     \mathrm{geometric:} & \mathcal{V}^\prime_{\geq2} & \supseteq & \mathcal{V}^\prime_{\geq3}  \\
     & \rotatebox{90}{$\supseteq$} & & \rotatebox{90}{$\supseteq$}\\
    \mathrm{algebraic:} & \V(\K^\prime_2)&\supseteq & \V(\K^\prime_3)  \\
    };
    \end{codi}
}
\]
\caption{Stratification framework II.}
\label{research framework_}
\end{figure}
Similarly to the role of the framework in Figure~\ref{research framework}, the derivations in the following subsections are based on the framework from Figure ~\ref{research framework_}.

\subsection{\texorpdfstring{$\mathcal{V}^\prime_{\ge 2}\subseteq\V(\K'_2)$}{}}
\subsubsection{\texorpdfstring{Decomposing $\V(\K'_2)$}{}}
Since $\mathcal{V}_{\ge 2}=\V(\J_1)$, let $\K_2=\J_1$ and compute
\begin{equation}
    \K_2^\prime=(\F+\J_1)\cap k[e_1^\prime,e_2^\prime,e_3^\prime,s_{12},s_{13},s_{23}].
    \nonumber
\end{equation}
However, due to the large number of variables and the high degree of polynomials, direct symbolic computation of $\K^\prime_2$ is extremely expensive. Therefore, a series of numerical experiments were carried out with fixed values of $s_{12},s_{13},s_{23}$.

For example, let $(s_{12},s_{13},s_{23})=(7,6,5)$ define the case $\triangle ABC$ of general acute triangles, then $\K^\prime_2$ is computed as
\begin{equation}
    \begin{aligned}
    \K^\prime_2=&\left \langle 54045009375000e_1'^{25}e_2'^{13}e_3'-134561043750000e_1'^{25}e_2'^{11}e_3'^{3}\right. \\
    &\left.+219894084375000e_1'^{25}e_2'^{9}e_3'^{5}-236144475000000e_1'^{25}e_2'^{7}e_3'^{7}\right.\\
    &\left.+193587975000000e_1'^{25}e_2'^{5}e_3'^{9}+\cdots\text{1167 terms}\cdots \right. \\
    &\left.-149079662555165224535534765625000000000000e_1'e_2'e_3'\right \rangle.
\end{aligned}
\nonumber
\end{equation}
The prime decomposition of its radical ideal is computed as
\begin{equation}
    \begin{aligned}
    \sqrt{\K^\prime_2}=&\left \langle e_1^\prime \right \rangle\cap\left \langle e_2^\prime \right \rangle\cap\left \langle e_3^\prime \right \rangle\cap \\
    &\left \langle 5 e_1'^{2} - 2 e_1' e_2' + 5 e_2'^{2} - 245 \right \rangle\cap\left \langle 5 e_1'^{2} + 2 e_1' e_2' + 5 e_2'^{2} - 245 \right \rangle\cap\\
    &\left \langle  35 e_1'^{2} - 38 e_1' e_3' + 35 e_3'^{2} - 1260\right \rangle \cap\left \langle  35 e_1'^{2} + 38 e_1' e_3' + 35 e_3'^{2} - 1260\right \rangle\cap\\
    &\left \langle 7 e_2'^{2} - 10 e_2' e_3' + 7 e_3'^{2} - 175  \right \rangle \cap \left \langle 7 e_2'^{2} + 10 e_2' e_3' + 7 e_3'^{2} - 175  \right \rangle\cap\\
    &\left \langle 25 e_1'^{4}- 12 e_1'^{2} e_2'^{2}- 38 e_1'^{2} e_3'^{2}+ 36 e_2'^{4}- 60 e_2'^{2} e_3'^{2}+ 49 e_3'^{4}- 44100 \right \rangle\cap \\
    &\left \langle 57624e_1'^{8}e_2'^{8}-141120e_1'^{8}e_2'^{6}e_3'^{2}+171072e_1'^{8}e_2'^{4}e_3'^{4}-103680e_1'^{8}e_2'^{2}e_3'^{6}\right. \\
    &\left.+31104e_1'^{8}e_3'^{8} +\cdots\text{98 terms} \cdots+3560485217625000e_3'^{2} - 26265874556250000\right \rangle.
\end{aligned}
\nonumber
\end{equation}

Clearly, the first three ideals correspond to trivial degenerate cases. The following six ideals geometrically correspond to three apple surfaces and three lemon surfaces associated with three sides of $\triangle ABC$, which are the Zariski closure of six arcs in $\pi(\V(\F+\J_{1}))$.
The tenth ideal actually corresponds to the danger cylinder of $\triangle ABC$. Consequently,
\begin{equation}
    \begin{aligned}
    \mathcal{V}^\prime_{\ge 2}=&\V\left(\left \langle C \right \rangle\right) \\
    =&\V\left(\left \langle 57624e_1'^{8}e_2'^{8}-141120e_1'^{8}e_2'^{6}e_3'^{2}+171072e_1'^{8}e_2'^{4}e_3'^{4}-103680e_1'^{8}e_2'^{2}e_3'^{6}\right.\right. \\
    &\left.+31104e_1'^{8}e_3'^{8} +\cdots\text{98 terms}\cdots+3560485217625000e_3'^{2} - 26265874556250000 \right \rangle).
\end{aligned}
\nonumber
\end{equation}
\begin{rmk}
    We choose the general acute triangle as a representative case to investigate the geometry of $\mathcal{V}^\prime_{\ge 2}$ and $\mathcal{V}^\prime_{\ge 3}$. The complete case study including six other representative triangles is provided in \ref{app Case Study}.
\end{rmk}

\subsubsection{\texorpdfstring{Interpreting $\mathcal{V}^\prime_{\ge 2}$}{}}

We continue to use the coordinate system (\ref{coordinate}) with $(s_{12},s_{13},s_{23})=(7,6,5)$. By the same transformation (\ref{edge}) for substituting $(e_1^\prime, e_2^\prime, e_3^\prime)$ in $C$, we derive a polynomial with total degree $12$ as
\begin{equation}
    \begin{aligned}
    D=&1719926784X^{10} Y^{2}+1146617856 X^{10} Z^{2}+4299816960X^{8} Y^{4}+\\
    &1719926784 X^{8} Z^{4}+\cdots \text{186 terms} \cdots+185980923555840 X^{2} Y^{2}Z^{2},
\end{aligned}
\nonumber
\end{equation}
which is exactly the deltoidal surface recently studied by Rieck \cite{rieck2025} and Wang et al. \cite{10.1007/s10851-021-01062-y}. Please see these works for the relationship between the deltoidal surface and the multi-solution phenomenon of the P3P problem. Here we investigate the geometric relationship between the danger cylinder and the deltoidal surface.
\begin{figure}[H]
\centering
\includegraphics[scale=0.15]{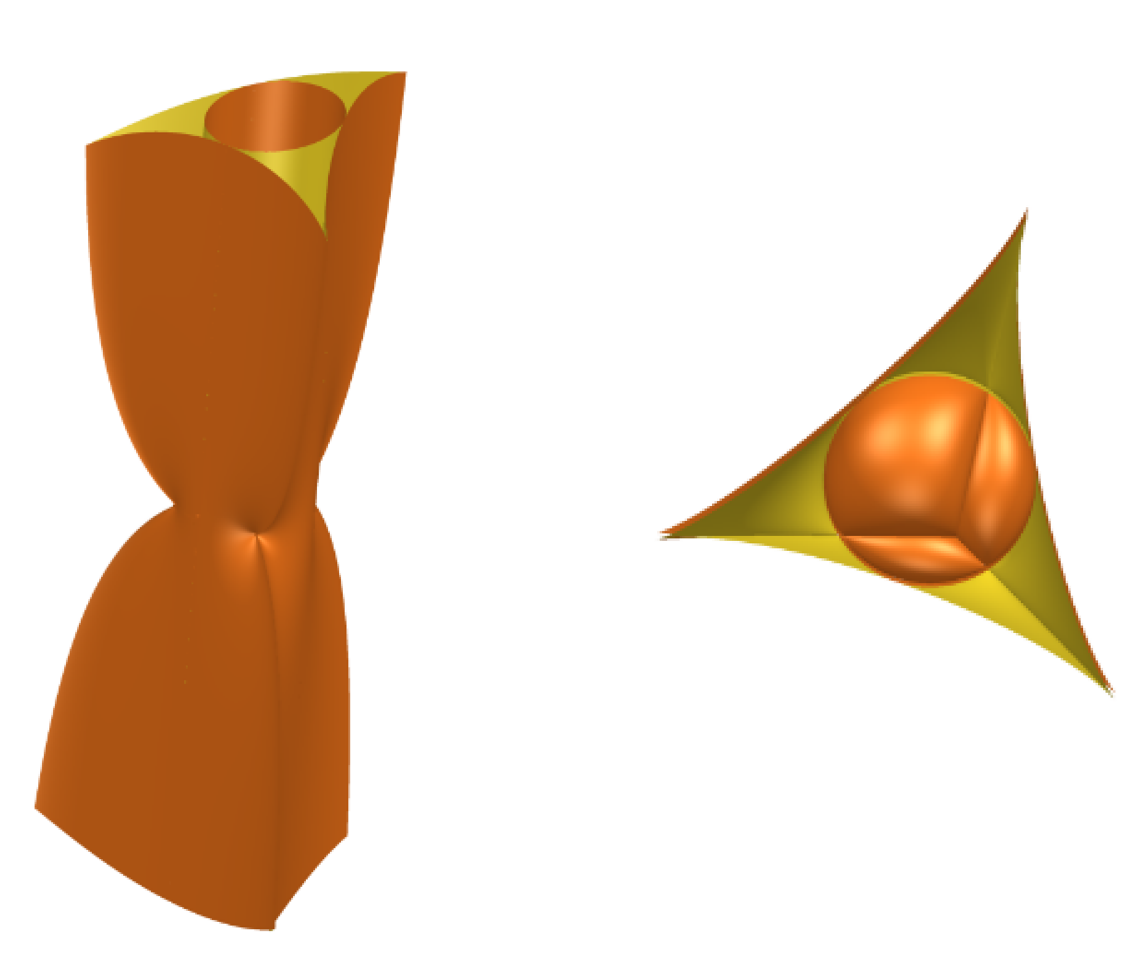}

\caption{The geometry of $\mathcal{V}'_{\geq2}$ and $\mathcal{V}'_{\geq3}$.}
\label{DSDC}
\end{figure}

\begin{thm}\label{deltoidal}
    When the singular configuration $O$ lies on the danger cylinder, its complementary solution $O'$ lies on a deltoidal surface. Moreover, the deltoidal surface is tangent to the danger cylinder along three Morley generatrices and intersects the danger cylinder at the circumcircle (see Figure~\ref{DSDC}).
\end{thm}
\begin{proof}
   Let $\mathcal{G}=\left \langle \nabla \J_1 \times\nabla C\right \rangle+ \J_1 $, then the following computation
    \begin{equation}
        1\in \mathcal{G}:\I_1^\infty~~\text{and}~~1\in \I_1:\mathcal{G}^\infty
        \nonumber
    \end{equation}
    imply $\V(\mathcal{G})=\V(\I_1)$, which concludes the tangency statement.

    Let $\mathcal{G}=\left \langle C \right \rangle+\J_1 $, $\mathcal{H}=\I_1\cap \I_2$, then the following computation
    \begin{equation}
        1\in \mathcal{G} : {\mathcal{H}}^{\infty}~~\text{and}~~1\in \mathcal{H}:{\mathcal{G}}^{\infty}
        \nonumber
    \end{equation}
    imply $\V(\mathcal{G})=\V(\mathcal{H})$, i.e.,$\V(\mathcal{G})=\V(\I_1)\cup\V(\I_2)$, which concludes the intersection statement.
\end{proof}

In other words, it is impossible to have two double P3P solutions.

\subsection{\texorpdfstring{$\mathcal{V}^\prime_{\ge 3}\subseteq\V(\K'_3)$}{}}
Since $\mathcal{V}_{\ge 3}=\V(\I_1)\cup\V(\I_2)$ and $\V(\I_2)$ correspond to infinitely many solutions of (\ref{tetra_form}), let $\K_3=\I_1$ and compute
\begin{equation}
    \K_3^\prime=(\F+\I_1)\cap k[e_1^\prime,e_2^\prime,e_3^\prime,s_{12},s_{13},s_{23}].
    \nonumber
\end{equation}
However, even if the values of $s_{12},s_{13},s_{23}$ have been fixed, the complete decomposition of $\K^\prime_3$ is still expensive. Therefore, we compute a superset of $\mathcal{V}^\prime_{\ge 3}$ alternatively.
\begin{thm}\label{cusp}
    When the singular configuration $O$ lies on one of three Morley generatrices, its complementary solution $O'$ lies on one of three cuspidal curves of the deltoidal surface (see Figure~\ref{DSDC}).
\end{thm}
\begin{proof}
    Let $C\in k[e_1^\prime,e_2^\prime,e_3^\prime]$, define $\mathcal{V}^\prime_{\ge 2}=\V\left(\left \langle C \right \rangle\right)$ and
    \begin{equation}
    \mathcal{C}=\left \langle C, \ \frac{\partial C}{\partial e_1^\prime},\ \frac{\partial C}{\partial e_2^\prime},\ \frac{\partial C}{\partial e_3^\prime}  \right \rangle ,
    \nonumber
\end{equation}
define the singular locus, namely, three cuspidal curves of the deltoidal surface, then the following computation
\begin{equation}
    1\in \K^\prime_3:\mathcal{C}^{\infty},
    \nonumber
\end{equation}
implies $\V(\K^\prime_3)\subseteq\V(\mathcal{C})$, which concludes $\mathcal{V}^\prime_{\ge 3}\subseteq\V(\mathcal{C})$.
\end{proof}

We perform a numerical experiment to illustrate the statement in Theorem~\ref{cusp}. Let $O$ denote a singular configuration located on a circle that is parallel to the circumcircle, and $O'$ denote its complementary configuration. As illustrated in Figure~\ref{numerical of Deltoid}, when $O$ moves along the circle inside the danger cylinder, $O'$ moves along the deltoidal surface (points of the same color are solutions of the same P3P equation system, see Figure~\ref{numerical of Deltoid}(a)). In particular, when $O$ lies on one of three Morley generatrices of the danger cylinder, $O'$ lies on one of three cuspidal curves of the deltoidal surface. Note that the trajectory of $O'$ is not planar (see Figure~\ref{numerical of Deltoid}(b)).

\begin{figure}[H]
\centering  
\subfigure[Top view]{
\includegraphics[height=0.33\textwidth]{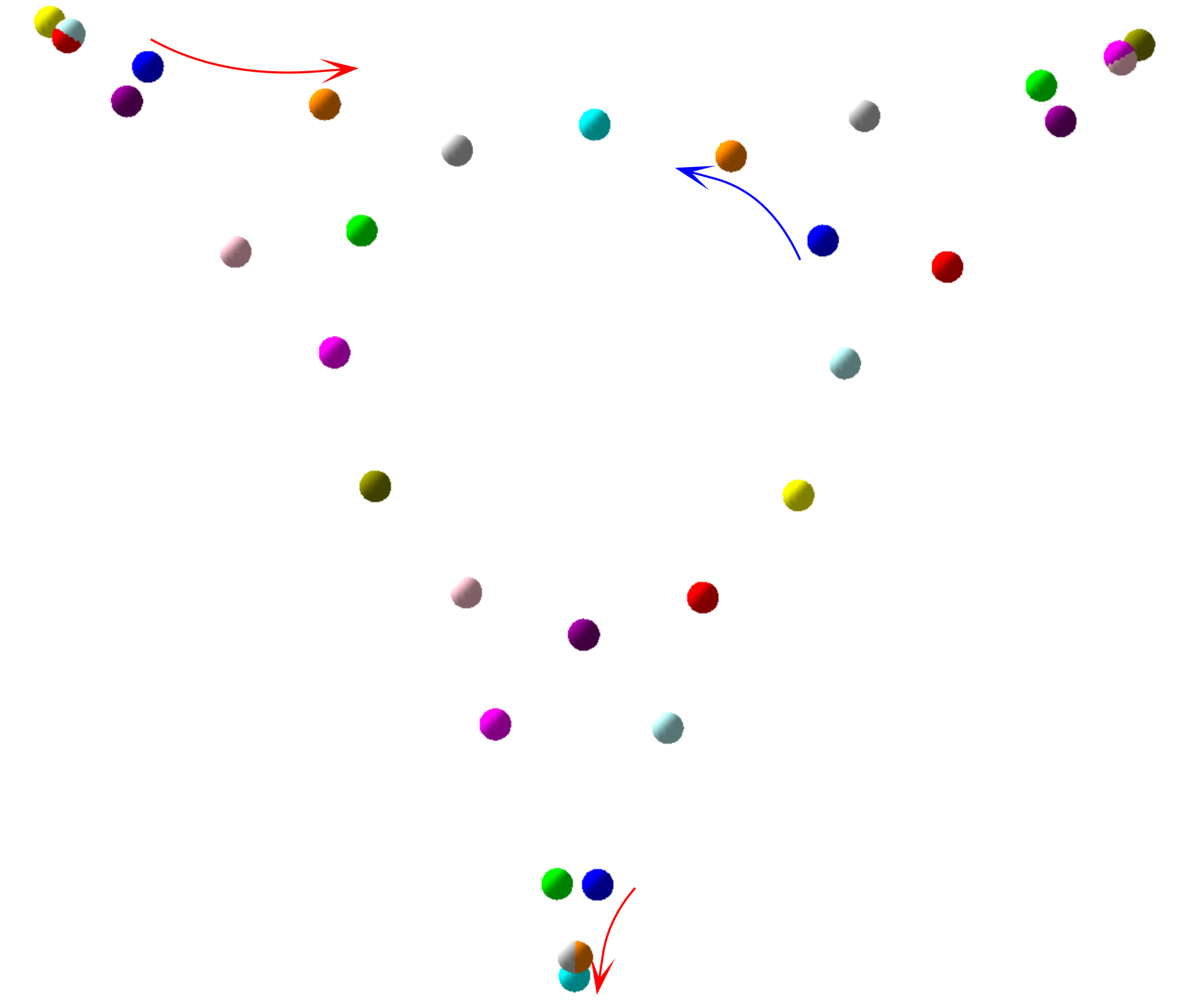}}
\subfigure[Front View]{
\includegraphics[height=0.33\textwidth]{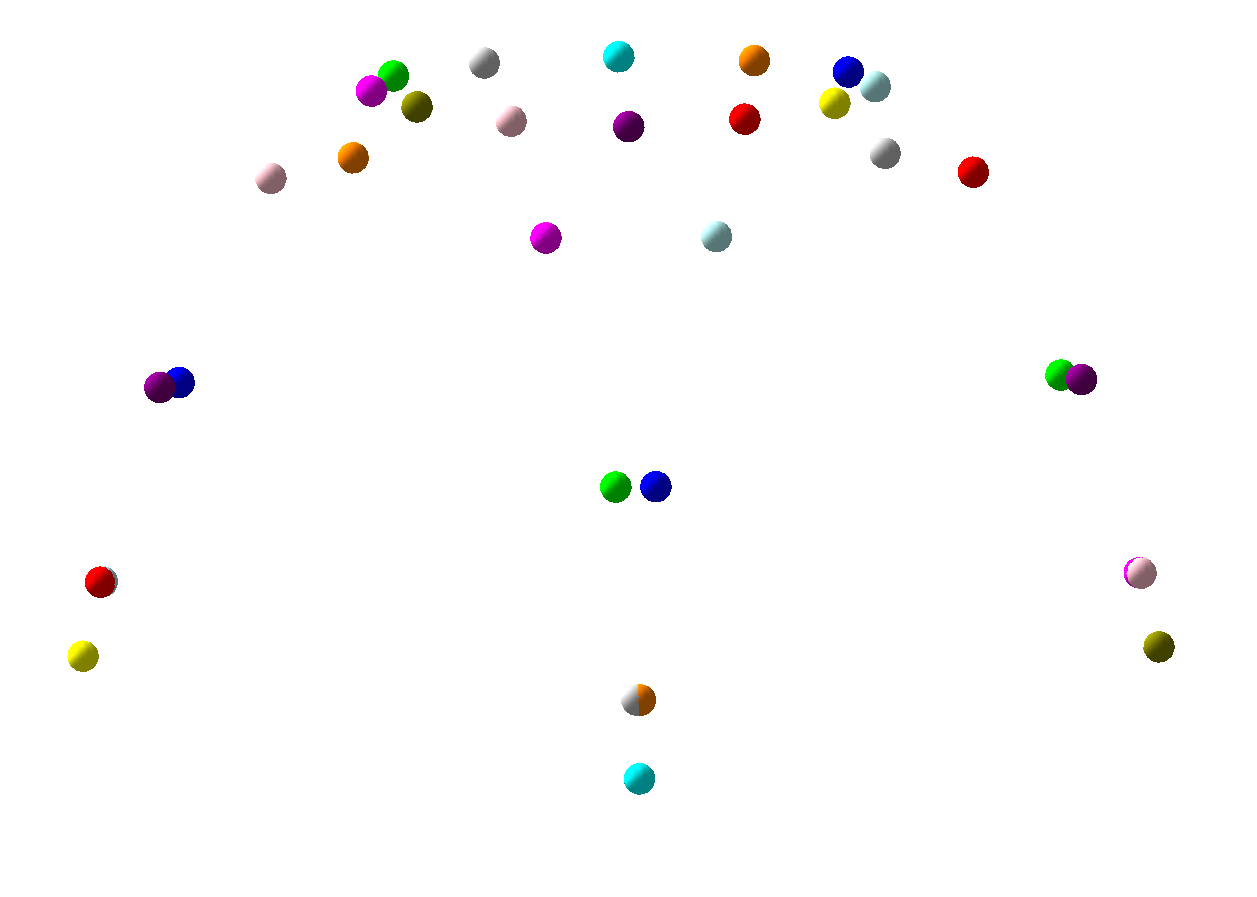}}
\caption{Numerical experiments of $\mathcal{V}_{\geq3}\subseteq\mathcal{V}_{\geq2}$ and $\mathcal{V}'_{\geq3}\subseteq\mathcal{V}'_{\geq2}$.}
\label{numerical of Deltoid}
\end{figure}

\subsection{Answer to Question~\ref{Q2}}
Recall the question: given $\triangle ABC$, suppose that the camera center $O$ corresponds to a singular solution $(e_1,e_2,e_3)$ of (\ref{tetra_form}) with multiplicity $\geq\mu$, then where are the other centers $O'$ complementary to $O$? We summarize the answer as follows.
\begin{itemize}
    \item For $O\in\mathcal{V}_{\ge 2}$, $O'$ lies on \textbf{the deltoidal surface} associated with the danger cylinder,
    \item For $O\in\mathcal{V}_{\ge 3}$, $O'$ lies on one of \textbf{three cuspidal curves} of the deltoidal surface.
\end{itemize}

Note that Theorems~\ref{deltoidal} and \ref{cusp} are true for six other representative triangles presented in \ref{app Case Study}.

\section{Conclusion and Discussion}\label{conclusion}

Using local dual space, a complete geometric stratification is given for  singular configurations of the P3P problem with respect to the multiplicity.
Using elimination and case study, we also provide a geometric stratification for complementary configurations associated with singular configurations.
Both results are summarized in Table~\ref{tab:summary_en}. Future work includes generalizing the proposed stratification framework to other PnP problems \cite{pascual2021complete}, related PnL problems \cite{jorge2022}, and analyzing kinematic singularity problems  in robotics, e.g.,  mechanism's higher-order singularities properties \cite{Nawratil24,Nawratil25}, singular configurations \cite{SPARTALIS2022102150}. This work thereby connects the singularity analysis of vision problems to the characterization required for complex robotic mechanisms.

\begin{table}[H]
  \centering
  \caption{Geometric stratification for singular configurations $O$ and complementary configurations $O'$ of the P3P problem with respect to the multiplicity $\mu$.}
  \label{tab:summary_en}
  \begin{tabular}{p{0.5cm} p{5cm} p{4.5cm}}
    \toprule
    {$\mu$} & {Location of $O$} & {Location of $O'$} \\
    \midrule
    $2$ & on  \textbf{the danger cylinder}  & on  \textbf{the deltoidal surface} \\
    \addlinespace
    $3$ & on \textbf{three Morley generatrices} of the danger cylinder  & on  \textbf{three cuspidal curves} of the deltoidal surface \\
    \addlinespace
    $4$ & \textbf{not exist}  & not defined \\
    \addlinespace
    $\infty$ & on \textbf{the circumcircle}  & not defined \\
    \bottomrule
  \end{tabular}
\end{table}

\bibliography{references}

\newpage
\appendix
\section{\texorpdfstring{Discover $\I_1$ and $\I_4$ in Section~\ref{sec:I12}}{Section 3.2.1}}
\label{app I1}
As we know from Section~\ref{sec:I12} that $\I_2,\I_3$ are degenerate cases, we would like to remove them from $\V(\sqrt{\J_1+\J_2})$. In algebraic geometry terms, removing one variety from the other amounts to finding the difference of the varieties\cite{cox1997ideals}.

The set difference of two affine varieties is generally not an affine variety but an open subset of a variety. It cannot be written as the set of solutions of a system of polynomial equations (it is not an affine variety). The smallest affine variety which contains it, is called the Zariski closure of the difference, denoted with an overline. Loosely speaking, taking the Zariski closure amounts to patching up the holes in the open set. Therefore, in this case, we need to find the following Zariski closure of the difference:
    \begin{equation}
        \V(\I_1^{\prime})=\overline{\V(\sqrt{\J_1+\J_2})\setminus (\V(\I_2)\cup\V(\I_3))}.
        \nonumber
    \end{equation}
     \begin{lem}[Corollary 11 of \S 4.4 \cite{cox1997ideals}]
       Let $\I$ and $\J$ be ideals in $k[x_1,\dots,x_n]$. If $k$ is algebraically closed and $\I$ is radical, then
       \begin{equation}
           \V(\I:\J)=\overline{\V(\I)\setminus \V(\J)}.
           \nonumber
       \end{equation}
   \end{lem}

   Therefore, by applying the lemma above, the quotient ideal $\sqrt{\J_1+\J_2}:(\I_2\cap\I_3)$ can be computed to obtain $\I^{\prime}_1$:

   \begin{equation}
    \begin{aligned}
    \I_1^{\prime}=&\left \langle e_{1}^{4} s_{12}^{2}-e_{1}^{4}s_{13}^{2}+2 e_{1}^{2} e_{2}^{2} s_{13}^{2}-2 e_{1}^{2} e_{3}^{2} s_{12}^{2}-e_{2}^{4} s_{13}^{2}+e_{2}^{2} s_{12}^{2} s_{13}^{2}+e_{3}^{4} s_{12}^{2}-e_{3}^{2}s_{12}^{2} s_{13}^{2}\ ,
    \right.\\
    &\left. -e_{1}^{4} s_{13}s_{23}^{2}+2 e_{1}^{2} e_{2}^{2} s_{13}s_{23}^{2}+e_{1}^{2} s_{12}^{2} s_{13} s_{23}^{2}+e_{2}^{4} s_{12}^{2} s_{13}
    \right.\\
    &\left.-e_{2}^{4} s_{13} s_{23}^{2}-2 e_{2}^{2} e_{3}^{2} s_{12}^{2} s_{13}+e_{3}^{4} s_{12}^{2} s_{13}-e_{3}^{2}s_{12}^{2} s_{13} s_{23}^{2}\ ,
    \right.\\
    &\left.e_{1}^{4} s_{12} s_{23}^{2}-2 e_{1}^{2} e_{3}^{2} s_{12} s_{23}^{2}-e_{3}^{4} s_{12} s_{13}^{2} +e_{3}^{4} s_{12} s_{23}^{2}\ ,\right. \\
    &\left. -e_{1}^{2}s_{12} s_{13}^{2} s_{23}^{2}-s_{12} e_{2}^{4} s_{13}^{2}+2 e_{2}^{2} e_{3}^{2} s_{12} s_{13}^{2}+e_{2}^{2} s_{12} s_{13}^{2} s_{23}^{2}-e_{1}^{4} s_{23}^{2}-e_{1}^{2} e_{2}^{2} s_{12}^{2}\right.\\
    &\left. +e_{1}^{2} e_{2}^{2} s_{13}^{2}+e_{1}^{2} e_{2}^{2} s_{23}^{2}+e_{1}^{2} e_{3}^{2} s_{12}^{2}-e_{1}^{2} e_{3}^{2} s_{13}^{2}+e_{1}^{2} e_{3}^{2} s_{23}^{2}-e_{2}^{4} s_{13}^{2}+e_{2}^{2} s_{12}^{2} e_{3}^{2} \right.\\
    &\left.+e_{2}^{2} e_{3}^{2} s_{13}^{2}-e_{2}^{2}e_{3}^{2} s_{23}^{2}-e_{3}^{4} s_{12}^{2}+s_{12}^{2} s_{13}^{2} s_{23}^{2}\ , \right.\\
    &\left.e_{1}^{6} s_{23}^{2}-2 e_{1}^{4} e_{2}^{2} s_{23}^{2}-e_{1}^{4} e_{3}^{2} s_{23}^{2}-e_{1}^{4} s_{12}^{2} s_{23}^{2}-e_{1}^{2} e_{2}^{4} s_{12}^{2}+e_{1}^{2}e_{2}^{4} s_{23}^{2}\right.\\
    &\left.+2 e_{1}^{2} e_{2}^{2} e_{3}^{2} s_{12}^{2}+2 e_{1}^{2} e_{2}^{2} e_{3}^{2} s_{23}^{2}-e_{1}^{2} e_{3}^{4} s_{12}^{2}+2 e_{1}^{2} e_{3}^{2} s_{12}^{2} s_{23}^{2}+e_{2}^{4} e_{3}^{2} s_{12}^{2}\right. \\
    &\left. -e_{2}^{4} e_{3}^{2} s_{23}^{2}-2 e_{2}^{2} e_{3}^{4} s_{12}^{2}+e_{3}^{6} s_{12}^{2}-e_{3}^{4} s_{12}^{2}s_{23}^{2}\ , \right.\\
    &\left.e_1^{6} s_{23}^{2}+e_1^{4} e_2^{2} s_{23}^{2}+2 e_1^{4} e_3^{2} s_{23}^{2}+e_1^{4} s_{13}^{2} s_{23}^{2}+e_1^{2} e_2^{4} s_{13}^{2}-2 e_1^{2} e_2^{2} e_3^{2} s_{13}^{2}\right. \\
    &\left.-2 e_1^{2} e_2^{2} e_3^{2} s_{23}^{2}-2 e_1^{2} e_2^{2} s_{23}^{2} s_{13}^{2}+e_1^{2} e_3^{4} s_{13}^{2}-e_1^{2} e_3^{4} s_{23}^{2}-e_2^{6} s_{13}^{2}+2 e_2^{4} e_3^{2} s_{13}^{2}\right.\\
    &\left.+e_2^{4} s_{13}^{2} s_{23}^{2}-e_2^{2} e_3^{4} s_{13}^{2}+e_2^{2} e_3^{4} s_{23}^{2}
    \right \rangle.
\end{aligned}
\nonumber
\end{equation}

$\I^{\prime}_1$ is a radical ideal, and its prime decomposition yields:
\begin{equation}
    \begin{aligned}
    \I_1^{\prime}=&\left \langle s_{12},-e_1+e_2\right \rangle  \cap \left \langle s_{12},e_1+e_2 \right \rangle \cap \left \langle s_{13}, -e_1+e_3 \right \rangle \\
    &\left \langle s_{13},e_1+e_3 \right \rangle \cap \left \langle s_{12}, s_{13}, s_{23} \right \rangle \cap \left \langle -e_1^{4} s_{12}^{2}+e_1^{4} s_{13}^{2}-2 e_1^{2} e_2^{2} s_{13}^{2}\right.\\
    &\left.+2 e_1^{2} e_3^{2} s_{12}^{2}+e_2^{4} s_{13}^{2}-e_2^{2} s_{12}^{2} s_{13}^{2}-e_3^{4} s_{12}^{2}+e_3^{2} s_{12}^{2} s_{13}^{2}\ , \right.\\
    &\left.e_1^{4} s_{23}^{2}-2 e_1^{2} e_2^{2} s_{23}^{2}-e_1^{2} s_{12}^{2} s_{23}^{2}-e_2^{4} s_{12}^{2}+e_2^{4} s_{23}^{2}+2 e_2^{2} s_{12}^{2} e_3^{2}-e_3^{4} s_{12}^{2}+e_3^{2} s_{12}^{2} s_{23}^{2}\ , \right.\\
    &\left.e_1^{4} s_{23}^{2}-2 e_1^{2} e_3^{2} s_{23}^{2}-s_{13}^{2} e_1^{2} s_{23}^{2}-e_2^{4} s_{13}^{2}+2 e_2^{2} e_3^{2} s_{13}^{2}+e_2^{2} s_{13}^{2} s_{23}^{2}-e_3^{4} s_{13}^{2}+e_3^{4} s_{23}^{2}\ , \right.\\
    &\left.-e_1^{4} s_{23}^{2}-e_1^{2} e_2^{2} s_{12}^{2}+e_1^{2} e_2^{2} s_{13}^{2}+e_1^{2} e_2^{2} s_{23}^{2}+e_1^{2} e_3^{2} s_{12}^{2}-e_1^{2} e_3^{2} s_{13}^{2}+e_1^{2} e_3^{2} s_{23}^{2}-e_2^{4} s_{13}^{2}\right.\\
    &\left.+e_2^{2} s_{12}^{2} e_3^{2}+e_2^{2} e_3^{2} s_{13}^{2}-e_2^{2} e_3^{2} s_{23}^{2}-e_3^{4} s_{12}^{2}+s_{12}^{2} s_{13}^{2} s_{23}^{2} \right \rangle ,
\end{aligned}
\nonumber
\end{equation}
   Since $s_{12}\neq 0,s_{13}\neq 0,s_{23}\neq 0$, we eliminate the degenerate cases from it to obtain $\I_1$:

   \begin{equation}
    \begin{aligned}
    \I_1=
    &\left \langle -e_1^{4} s_{12}^{2}+e_1^{4} s_{13}^{2}-2 e_1^{2} e_2^{2} s_{13}^{2}+2 e_1^{2} e_3^{2} s_{12}^{2}+e_2^{4} s_{13}^{2}-e_2^{2} s_{12}^{2} s_{13}^{2}-e_3^{4} s_{12}^{2}+e_3^{2} s_{12}^{2} s_{13}^{2}\ ,\right.\\
    &\left. e_1^{4} s_{23}^{2}-2 e_1^{2} e_2^{2} s_{23}^{2}-e_1^{2} s_{12}^{2} s_{23}^{2}-e_2^{4} s_{12}^{2}+e_2^{4} s_{23}^{2}+2 e_2^{2} s_{12}^{2} e_3^{2}-e_3^{4} s_{12}^{2}+e_3^{2} s_{12}^{2} s_{23}^{2}\ , \right.\\
    &\left.e_1^{4} s_{23}^{2}-2 e_1^{2} e_3^{2} s_{23}^{2}-s_{13}^{2} e_1^{2} s_{23}^{2}-e_2^{4} s_{13}^{2}+2 e_2^{2} e_3^{2} s_{13}^{2}+e_2^{2} s_{13}^{2} s_{23}^{2}-e_3^{4} s_{13}^{2}+e_3^{4} s_{23}^{2}\ ,\right.\\
    &\left.-e_1^{4} s_{23}^{2}-e_1^{2} e_2^{2} s_{12}^{2}+e_1^{2} e_2^{2} s_{13}^{2}+e_1^{2} e_2^{2} s_{23}^{2}+e_1^{2} e_3^{2} s_{12}^{2}-e_1^{2} e_3^{2} s_{13}^{2}+e_1^{2} e_3^{2} s_{23}^{2}-e_2^{4} s_{13}^{2}\right.\\
    &\left.+e_2^{2} s_{12}^{2} e_3^{2}+e_2^{2} e_3^{2} s_{13}^{2}-e_2^{2} e_3^{2} s_{23}^{2}-e_3^{4} s_{12}^{2}+s_{12}^{2} s_{13}^{2} s_{23}^{2} \right \rangle.
\end{aligned}
\nonumber
\end{equation}

Denote the intersection of the remaining ideals as $\I_4$.

\section{\texorpdfstring{Prove Morley's theorem and $f_1$ in Section~\ref{sec:vI1}}{in Section 3.2.2}}
\label{app Morley}


\subsubsection*{Morley's theorem}
\begin{proof}
Let $\triangle ABC$ be an arbitrary triangle with interior angles $3\alpha,3\beta,3\gamma$ satisfying $\alpha+\beta+\gamma=\frac{\pi}{3}$. The first Morley triangle $\triangle DEF$ is formed by the intersections of the internal angle trisectors adjacent to the sides, as shown in Figure~\ref{Morley triangle}(a).

Consider triangle $\triangle ADE$. By applying the Law of Sines successively in triangles $\triangle ABD, \triangle ABC$ and $\triangle AEC$, we derive the ratio:
According to the Law of Sines, we have:
\leavevmode
    \begin{align}
        \frac{AD}{AE}&=\frac{AD}{AB}\cdot \frac{AB}{AC}\cdot \frac{AC}{AE}\nonumber\\
        &=\frac{\sin{\gamma}}{\sin{(\alpha+\gamma)}}\cdot \frac{\sin{(3\beta)}}{\sin{(3\gamma)}}\cdot \frac{\sin{(\alpha+\beta)}}{\sin{\beta}} \nonumber\\
        &=\frac{\sin{(\frac{\pi}{3}+\beta)}}{\sin{(\frac{\pi}{3}+\gamma)}} \nonumber,
    \end{align}
    This implies that the angles of triangle $\triangle ADE$ satisfy:
    \begin{equation}
      \angle AED=\frac{\pi}{3}+\beta,\quad \angle ADE=\frac{\pi}{3}+\gamma .
      \nonumber
    \end{equation}

    Now, consider the angles around point $D$ in the Morley triangle configuration:
    \begin{equation}
        \angle DEF=2\pi-(\pi-\alpha-\beta)-(\frac{\pi}{3}+\beta)-(\frac{\pi}{3}+\alpha)=\frac{\pi}{3}.
        \nonumber
    \end{equation}

    Similarly, we can show:
    \begin{equation}
       \angle EDF=\frac{\pi}{3}, \quad  \angle DFE=\frac{\pi}{3}.
       \nonumber
    \end{equation}

    Therefore, all interior angles of $\triangle DEF$ equal $\frac{\pi}{3}$, proving that the Morley triangle is equilateral.
\end{proof}
\subsubsection*{Derivation of $f_1$}
As shown in Fig.\ref{Morley triangle}(b), denote the angle between $l_2$ and $DE$ as $\psi$. Angle derivation within the Morley triangle yields:
\begin{align}
    \psi&=\pi-2\beta-(\frac{\pi}{3}+\gamma)-\frac{\pi}{3} \nonumber \\
    &=\alpha-\beta \nonumber ,
\end{align}
Thus we obtain:
\begin{equation}
\theta=\psi+\frac{\pi}{6}=\alpha-\beta+\frac{\pi}{6}.
\nonumber
\end{equation}

The trigonometric expression for $\sin{(3\theta)}$ is computed as follows:
\begin{align}
    \sin{(3\theta)}&=\sin{(3\alpha-3\beta+\frac{\pi}{2})} \nonumber \\
    &=\cos{(3\alpha-3\beta)} \nonumber \\
    &=\cos{(3\alpha)}\cos{(3\beta)}+\sin{(3\alpha)}\sin{(3\beta)}, \nonumber
\end{align}
Applying the Law of Sines and the Law of Cosines to triangle $\triangle ABC$, we obtain the explicit form:
\begin{equation}
    \sin{(3\theta)}=\frac{-s_{13}^{4} + s_{12}^{2}s_{13}^{2} + 2s_{23}^{2}s_{13}^{2} + s_{12}^{2}s_{23}^{2} - s_{23}^{4}}{2s_{12}^{2}s_{13}s_{23}}, \nonumber
\end{equation}
Consequently, a cubic equation in $\sin{\theta}$ is obtained:
\begin{equation}
    8s_{12}^{2}s_{13}s_{23}\sin^3{\theta}-6s_{12}^{2}s_{13}s_{23}\sin{\theta}-s_{13}^{4} + s_{12}^{2}s_{13}^{2} + 2s_{23}^{2}s_{13}^{2} + s_{12}^{2}s_{23}^{2} - s_{23}^{4}=0.
    \nonumber
\end{equation}

\section{Case Study for Other Six Representative Triangles}
\label{app Case Study}
\subsection{Equilateral Triangle}
Let $(s_{12},s_{13},s_{23})=(1,1,1)$ define the case $\triangle ABC$ of equilateral triangle, then $\K^\prime_2$ is computed as
\begin{equation}
    \begin{aligned}
    \K^\prime_2=&\left \langle 3 e_1'^{25} e_2'^{13} e_3' - 3 e_1'^{25} e_2'^{11} e_3'^{3} + 6 e_1'^{25} e_2'^{9} e_3'^{5} - 3 e_1'^{25} e_2'^{7} e_3'^{7} + 6 e_1'^{25} e_2'^{5} e_3'^{9}\right.  \\
    &\left.+ 8 e_1'^{3} e_2' e_3' + 8 e_1' e_2'^{3} e_3'+\cdots \text{1144 terms} \cdots + 8 e_1' e_2' e_3'^{3} - e_1' e_2' e_3'\right \rangle.
\end{aligned}
\nonumber
\end{equation}
The prime decomposition of its radical ideal is computed as
\begin{equation}
    \begin{aligned}
    \sqrt{\K^\prime_2}=&\left \langle e_1^\prime \right \rangle\cap\left \langle e_2^\prime \right \rangle\cap\left \langle e_3^\prime \right \rangle\cap\\
    &\left \langle e_1'^{2} - e_2' e_1' + e_2'^{2} - 1\right \rangle\cap\left \langle e_1'^{2} + e_1' e_2' + e_2'^{2} - 1 \right \rangle \cap  \\
    &\left \langle e_1'^{2} - e_1' e_3' + e_3'^{2} - 1 \right \rangle \cap\left \langle e_1'^{2} + e_1' e_3' + e_3'^{2} - 1 \right \rangle \cap \\
    &\left \langle e_2'^{2} - e_2' e_3' + e_3'^{2}- 1 \right \rangle\cap\left \langle e_2'^{2} + e_2' e_3' + e_3'^{2} - 1 \right \rangle \cap\\
    &\left \langle e_1'^{4} - e_1'^{2} e_2'^{2} - e_1'^{2} e_3'^{2} + e_2'^{4} - e_2'^{2} e_3'^{2}+ e_3'^{4} - 1 \right \rangle \cap\\
    & \left \langle 3 e_1'^{8} e_2'^{8} - 6 e_1'^{8} e_2'^{6} e_3'^{2} + 9 e_1'^{8} e_2'^{4} e_3'^{4} - 6 e_1'^{8} e_2'^{2} e_3'^{6} \right.  \\
    &\left. + 3 e_1'^{8} e_3'^{8}+\cdots \text{83 terms} \cdots -6 e_3'^{4} + 4 e_1'^{2} + 4 e_2'^{2} + 4 e_3'^{2} - 1\right \rangle .
\end{aligned}
\nonumber
\end{equation}

Remove the degenerate cases, it follows that:
\begin{equation}
    \begin{aligned}
    \mathcal{V}^\prime_{\ge2}=&\V(\left \langle C \right \rangle)\\
    =&\V( \left \langle 3 e_1'^{8} e_2'^{8} - 6 e_1'^{8} e_2'^{6} e_3'^{2} + 9 e_1'^{8} e_2'^{4} e_3'^{4} - 6 e_1'^{8} e_2'^{2} e_3'^{6} \right.  \\
    &\left. + 3 e_1'^{8} e_3'^{8}+\cdots \text{83 terms} \cdots -6 e_3'^{4} + 4 e_1'^{2} + 4 e_2'^{2} + 4 e_3'^{2} - 1\right \rangle ).
\end{aligned}
\nonumber
\end{equation}

\subsection{Isosceles Right Triangle}
Let $(s_{12},s_{13},s_{23})=(\sqrt{2},1,1)$ define the case $\triangle ABC$ of isosceles right triangle, then $\K^\prime_2$ is computed as
\begin{equation}
    \begin{aligned}
    \K^\prime_2=&\left \langle4 e_1'^{23} e_2'^{13} e_3' - 8 e_1'^{23} e_2'^{11} e_3'^{3} + 12 e_1'^{23} e_2'^{9} e_3'^{5} - 12 e_1'^{23} e_2'^{7} e_3'^{7} + 9 e_1'^{23} e_2'^{5} e_3'^{9}\right. \\
    &\left.- 176 e_1' e_2'^{3} e_3'+\cdots \text{977 terms}\cdots - 256 e_1' e_2' e_3'^{3} + 32 e_1' e_2' e_3'\right \rangle .
\end{aligned}
\nonumber
\end{equation}
The prime decomposition of its radical ideal is computed as
\begin{equation}
    \begin{aligned}
    \sqrt{\K^\prime_2}=&\left \langle e_1^\prime \right \rangle\cap\left \langle e_2^\prime \right \rangle\cap\left \langle e_3^\prime \right \rangle\cap\\
    &\left\langle e_1'^{2} + e_2'^{2} - 2 \right \rangle \cap\\
    &\left \langle  e_1'^{4} + e_3'^{4} - 2e_1'^{2} - 2e_3'^{2} + 1\right \rangle\cap \left \langle e_2'^{4} + e_3'^{4} - 2e_2'^{2} - 2e_3'^{2} + 1 \right \rangle\cap\\
    &\left \langle e_1'^{4} - 2e_1'^{2}e_3'^{2} + e_2'^{4} - 2e_2'^{2}e_3'^{2} + 2e_3'^{4}- 2 \right \rangle\cap \\
    & \left \langle 4 e_1'^{8} e_2'^{8}- 8e_1'^{8} e_2'^{6} e_3'^{2}+ 8 e_1'^{8} e_2'^{4} e_3'^{4}- 4 e_1'^{8} e_2'^{2} e_3'^{6}+ e_1'^{8} e_3'^{8}\right. \\
    &\left.-24 e_3'^{4}+\cdots \text{73 terms}\cdots  + 12 e_1'^{2} + 12 e_2'^{2} + 16 e_3'^{2} - 4\right \rangle .
\end{aligned}
\nonumber
\end{equation}

Remove the degenerate cases, it follows that:
\begin{equation}
    \begin{aligned}
    \mathcal{V}^\prime_{\ge2}=&\V(\left \langle C \right \rangle)\\
    =&\V(\left \langle 4 e_1'^{8} e_2'^{8}- 8e_1'^{8} e_2'^{6} e_3'^{2}+ 8 e_1'^{8} e_2'^{4} e_3'^{4}- 4 e_1'^{8} e_2'^{2} e_3'^{6}+ e_1'^{8} e_3'^{8}\right. \\
    &\left.-24 e_3'^{4}+\cdots \text{73 terms}\cdots  + 12 e_1'^{2} + 12 e_2'^{2} + 16 e_3'^{2} - 4\right \rangle ).
\end{aligned}
\nonumber
\end{equation}

\subsection{Isosceles Acute Triangle}
Let $(s_{12},s_{13},s_{23})=(4,3,3)$ define the case $\triangle ABC$ of isosceles acute triangle, then $\K^\prime_2$ is computed as
\begin{equation}
    \begin{aligned}
    \K^\prime_2=&\left \langle 680244480 e_1'^{25} e_2'^{13} e_3'
- 1209323520 e_1'^{25} e_2'^{11} e_3'^{3}
+ 1823433120e_1'^{25} e_2'^{9} e_3'^{5}
\right.  \\
&\left.- 1804537440 e_1'^{25} e_2'^{7} e_3'^{7}
+\cdots\text{1156 terms}\cdots + 1490692005 e_1'^{25} e_2'^{5} e_3'^{9}\right.  \\
&\left.-815887117849349003163992064\,e_1'e_2'e_3'\right \rangle .
\end{aligned}
\nonumber
\end{equation}
The prime decomposition of its radical ideal is computed as
\begin{equation}
    \begin{aligned}
    \sqrt{\K^\prime_2}=&\left \langle e_1^\prime \right \rangle\cap\left \langle e_2^\prime \right \rangle\cap\left \langle e_3^\prime \right \rangle\cap\\
    &\left \langle 9 e_1'^{2} - 2 e_1' e_2' + 9 e_2'^{2} - 144 \right \rangle\cap\left \langle 9 e_1'^{2} + 2 e_1' e_2' + 9 e_2'^{2} - 144 \right \rangle\cap  \\
    &\left \langle  3 e_1'^{2} - 4 e_1' e_3' + 3 e_3'^{2} - 27\right \rangle\cap \left \langle  3 e_1'^{2} + 4 e_1' e_3' + 3 e_3'^{2} - 27\right \rangle\cap\\
    &\left \langle3 e_2'^{2} - 4 e_2' e_3' + 3 e_3'^{2} - 27\right \rangle\cap\left \langle3 e_2'^{2} + 4 e_2' e_3' + 3 e_3'^{2} - 27\right \rangle \cap \\
    &\left \langle  9 e_1'^{4}- 2 e_1'^{2} e_2'^{2}- 16 e_1'^{2} e_3'^{2}+ 9 e_2'^{4}- 16 e_2'^{2} e_3'^{2}+ 16 e_3'^{4}- 1296\right \rangle \cap\\
    &\left \langle 1280 e_1'^{8} e_2'^{8}-2560 e_1'^{8} e_2'^{6} e_3'^{2}+2720 e_1'^{8} e_2'^{4} e_3'^{4}-1440 e_1'^{8} e_2'^{2} e_3'^{6}\right.  \\
    &\left.+405 e_1'^{8} e_3'^{8}+\cdots\text{95 terms} \cdots -3265173504e_3'^{4}+15305500800e_1'^{2}\right. \\
    &\left. +15305500800e_2'^{2}+19591041024e_3'^{2}-44079842304\right \rangle.
\end{aligned}
\nonumber
\end{equation}

Remove the degenerate cases, it follows that:
\begin{equation}
    \begin{aligned}
    \mathcal{V}^\prime_{\ge2}=&\V(\left \langle C \right \rangle)\\
    =&\V(\left \langle 1280 e_1'^{8} e_2'^{8}-2560 e_1'^{8} e_2'^{6} e_3'^{2}+2720 e_1'^{8} e_2'^{4} e_3'^{4}-1440 e_1'^{8} e_2'^{2} e_3'^{6}\right.  \\
    &\left.+405 e_1'^{8} e_3'^{8}+\cdots\text{95 terms} \cdots -3265173504e_3'^{4}+15305500800e_1'^{2}\right. \\
    &\left. +15305500800e_2'^{2}+19591041024e_3'^{2}-44079842304\right \rangle ) .
\end{aligned}
\nonumber
\end{equation}

\subsection{Isosceles Obtuse Triangle}
Let $(s_{12},s_{13},s_{23})=(5,3,3)$ define the case $\triangle ABC$ of isosceles obtuse triangle, then $\K^\prime_2$ is computed as
\begin{equation}
    \begin{aligned}
    \K^\prime_2=&\left \langle 3653656875e_1'^{25} e_2'^{13} e_3'-10149046875e_1'^{25} e_2'^{11} e_3'^{3}+15621412950e_1'^{25}e_2'^{9} e_3'^{5} \right.  \\
    &\left.-14825727675e_1'^{25} e_2'^{7} e_3'^{7}+\cdots\text{1156 terms}\cdots+8803851606e_1'^{25} e_2'^{5} e_3'^{9} \right. \\
    &\left.-74204501092876762701416015625e_1'e_2'e_3' \right \rangle.
\end{aligned}
\nonumber
\end{equation}
The prime decomposition of its radical ideal is computed as
\begin{equation}
    \begin{aligned}
     \sqrt{\K^\prime_2}=&\left \langle e_1^\prime \right \rangle\cap\left \langle e_2^\prime \right \rangle\cap\left \langle e_3^\prime \right \rangle\cap\\
     &\left \langle 9 e_1'^{2} - 7 e_1' e_2' + 9 e_2'^{2} - 225 \right \rangle \cap\left \langle 9 e_1'^{2} + 7 e_1' e_2' + 9 e_2'^{2} - 225 \right \rangle \cap\\
    &\left \langle 3 e_1'^{2} - 5 e_1' e_3' + 3 e_3'^{2} - 27 \right \rangle \cap \left \langle 3 e_1'^{2} + 5 e_1' e_3' + 3 e_3'^{2} - 27 \right \rangle\cap \\
    &\left \langle  3 e_2'^{2} - 5 e_2' e_3' + 3 e_3'^{2} - 27\right \rangle \cap \left \langle  3 e_2'^{2} + 5 e_2' e_3' + 3 e_3'^{2} - 27\right \rangle \cap \\
    &\left \langle9 e_1'^{4}+ 7 e_1'^{2} e_2'^{2}- 25 e_1'^{2} e_3'^{2}+ 9 e_2'^{4}- 25 e_2'^{2} e_3'^{2}+ 25 e_3'^{4}- 2025\right \rangle\cap  \\
    &\left \langle 6875e_1'^{8}e_2'^{8}-13750e_1'^{8}e_2'^{6}e_3'^{2}+11825e_1'^{8}e_2'^{4}e_3'^{4}-4950e_1'^{8}e_2'^{2}e_3'^{6}\right.  \\
    &\left.+891e_1'^{8}e_3'^{8} +\cdots\text{97 terms}\cdots+203276182500e_2'^{2}+298935562500e_3'^{2} \right. \\
    &\left. -672605015625\right \rangle.
\end{aligned}
\nonumber
\end{equation}

Remove the degenerate cases, it follows that:
\begin{equation}
    \begin{aligned}
    \mathcal{V}^\prime_{\ge2}=&\V(\left \langle C \right \rangle)\\
    =&\V(\left \langle 6875e_1'^{8}e_2'^{8}-13750e_1'^{8}e_2'^{6}e_3'^{2}+11825e_1'^{8}e_2'^{4}e_3'^{4}-4950e_1'^{8}e_2'^{2}e_3'^{6}\right. \\
    &\left. +891e_1'^{8}e_3'^{8}+\cdots\text{97 terms}\cdots +203276182500e_2'^{2}+298935562500e_3'^{2} \right.\\
    &\left. -672605015625\right \rangle).
\end{aligned}
\nonumber
\end{equation}

\subsection{General Right Triangle}
Let $(s_{12},s_{13},s_{23})=(5,4,3)$ define the case $\triangle ABC$ of general right triangle, then $\K^\prime_2$ is computed as
\begin{equation}
    \begin{aligned}
    \K^\prime_2=&\left \langle 31640625e_1'^{23}e_2'^{13}e_3'-98718750e_1'^{23}e_2'^{11}e_3'^{3}+169340625e_1'^{23}e_2'^{9}e_3'^{5}\right.  \\
    &\left.-184550400e_1'^{23}e_2'^{7}e_3'^{7}+\cdots\text{987 terms}\cdots+134330400e_1'^{23}e_2'^{5}e_3'^{9} \right.  \\
    &\left.
+1224440064000000000000000000e_1'e_2'e_3'\right \rangle.
\end{aligned}
\nonumber
\end{equation}
The prime decomposition of its radical ideal is computed as
\begin{equation}
    \begin{aligned}
    \sqrt{\K^\prime_2}=&\left \langle e_1^\prime \right \rangle\cap\left \langle e_2^\prime \right \rangle\cap\left \langle e_3^\prime \right \rangle\cap\\
    &\left \langle e_1'^{2} + e_2'^{2} - 25 \right \rangle\cap \\
    &\left \langle 5 e_1'^{2} - 6 e_1' e_3' + 5 e_3'^{2} - 80 \right \rangle \cap\left \langle 5 e_1'^{2} + 6 e_1' e_3' + 5 e_3'^{2} - 80 \right \rangle\cap  \\
    &\left \langle 5 e_2'^{2} - 8 e_2' e_3' + 5 e_3'^{2} - 45 \right \rangle\cap\left \langle 5 e_2'^{2} + 8 e_2' e_3' + 5 e_3'^{2} - 45 \right \rangle \cap \\
    &\left \langle 9 e_1'^{4}- 18 e_1'^{2} e_3'^{2}+ 16 e_2'^{4}- 32 e_2'^{2} e_3'^{2}+ 25 e_3'^{4}- 3600 \right \rangle \cap \\
    &\left \langle 625e_1'^{8}e_2'^{8}-1600e_1'^{8}e_2'^{6}e_3'^{2}+1824e_1'^{8}e_2'^{4}e_3'^{4}-1024e_1'^{8}e_2'^{2}e_3'^{6}\right.\\
    &\left.+256e_1'^{8}e_3'^{8}+\cdots\text{97 terms}\cdots +59778000000e_1'^{2}+88128000000e_2'^{2}\right.\\
    &\left.+101250000000e_3'^{2}-291600000000\right \rangle.
\end{aligned}
\nonumber
\end{equation}

Remove the degenerate cases, it follows that:
\begin{equation}
    \begin{aligned}
    \mathcal{V}^\prime_{\ge2}=&\V(\left \langle C \right \rangle)\\
    =&\V(\left \langle 625e_1'^{8}e_2'^{8}-1600e_1'^{8}e_2'^{6}e_3'^{2}+1824e_1'^{8}e_2'^{4}e_3'^{4}-1024e_1'^{8}e_2'^{2}e_3'^{6}\right. \\
    &\left.+256e_1'^{8}e_3'^{8}+\cdots\text{97 terms}\cdots +59778000000e_1'^{2}+88128000000e_2'^{2}\right.\\
    &\left.+101250000000e_3'^{2}-291600000000\right \rangle).
\end{aligned}
\nonumber
\end{equation}

\subsection{General Obtuse Triangle}
Let $(s_{12},s_{13},s_{23})=(7,5,3)$ define the case $\triangle ABC$ of general obtuse triangle, then $\K^\prime_2$ is computed as
\begin{equation}
    \begin{aligned}
   \K^\prime_2=&\left \langle 1400846643e_1'^{25}e_2'^{13}e_3'-5746330107e_1'^{25}e_2'^{11}e_3'^{3}+10680507558e_1'^{25}e_2'^{9}e_3'^{5}\right. \\
    &\left.-11255746635e_1'^{25}e_2'^{7}e_3'^{7}+\cdots\text{1150 terms}\cdots+7006674150e_1'^{25}e_2'^{5}e_3'^{9} \right. \\
    &\left.+1407974590798782676168939453125000\,e_1'e_2'e_3'^{3}
\right. \\
&\left.-2329369727424456633367730712890625\,e_1'e_2'e_3'\right \rangle .
\end{aligned}
\nonumber
\end{equation}
The prime decomposition of its radical ideal is computed as
\begin{equation}
    \begin{aligned}
    \sqrt{\K^\prime_2}=&\left \langle e_1^\prime \right \rangle\cap\left \langle e_2^\prime \right \rangle\cap\left \langle e_3^\prime \right \rangle\cap\\
    &\left \langle e_1'^{2} - e_1' e_2' + e_2'^{2} - 49 \right \rangle\cap\left \langle e_1'^{2} + e_1' e_2' + e_2'^{2} - 49 \right \rangle\cap \\
    &\left \langle  7 e_1'^{2} - 11 e_1' e_3' + 7 e_3'^{2} - 175\right \rangle\cap\left \langle  7 e_1'^{2} + 11 e_1' e_3' + 7 e_3'^{2} - 175\right \rangle\cap \\
    &\left \langle 7 e_2'^{2} - 13 e_2' e_3' + 7 e_3'^{2} - 63 \right \rangle\cap\left \langle 7 e_2'^{2} + 13 e_2' e_3' + 7 e_3'^{2} - 63 \right \rangle\cap\\
    &\left \langle 9 e_1'^{4}+ 15 e_1'^{2} e_2'^{2}- 33 e_1'^{2} e_3'^{2}+ 25 e_2'^{4}- 65 e_2'^{2} e_3'^{2}+ 49 e_3'^{4}- 11025 \right \rangle\cap \\
    &\left \langle 7203 e_1'^{8} e_2'^{8}-19110 e_1'^{8} e_2'^{6} e_3'^{2}+20025 e_1'^{8} e_2'^{4} e_3'^{4}-9750 e_1'^{8} e_2'^{2} e_3'^{6}\right. \\
    &\left.+1875 e_1'^{8} e_3'^{8} +\cdots \text{99 terms}\cdots+19845327442500e_3'^{2}-65664686390625\right \rangle.
\end{aligned}
\nonumber
\end{equation}

Remove the degenerate cases, it follows that:
\begin{equation}
    \begin{aligned}
     \mathcal{V}^\prime_{\ge2}=&\V(\left \langle C \right \rangle)\\
     =&\V(\left \langle 7203 e_1'^{8} e_2'^{8}-19110 e_1'^{8} e_2'^{6} e_3'^{2}+20025 e_1'^{8} e_2'^{4} e_3'^{4}-9750 e_1'^{8} e_2'^{2} e_3'^{6}\right. \\
    &\left. +1875 e_1'^{8} e_3'^{8}+\cdots \text{99 terms}\cdots+19845327442500e_3'^{2}\right. \\
    &\left.-65664686390625\right \rangle) .
\end{aligned}
\nonumber
\end{equation}

    \begin{figure}[H]
\centering
\subfigure[Equilateral triangle]{
\includegraphics[scale=0.25]{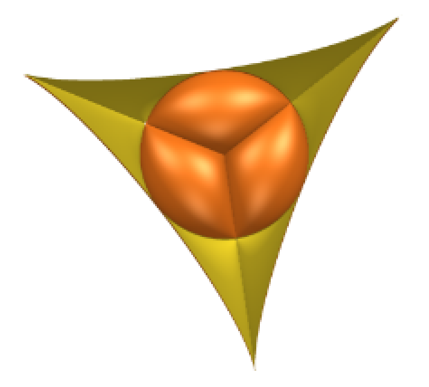}
}
\subfigure[Isosceles right triangle]{
\includegraphics[scale=0.25]{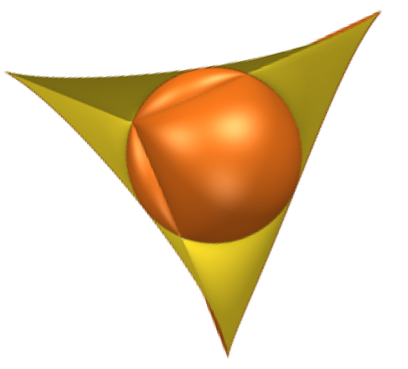}
}
\subfigure[Isosceles acute triangle]{
\includegraphics[scale=0.25]{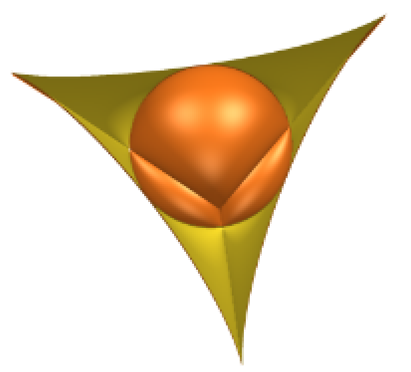}
}
\end{figure}

\begin{figure}[H]
\centering
\subfigure[Isosceles obtuse triangle]{
\includegraphics[scale=0.25]{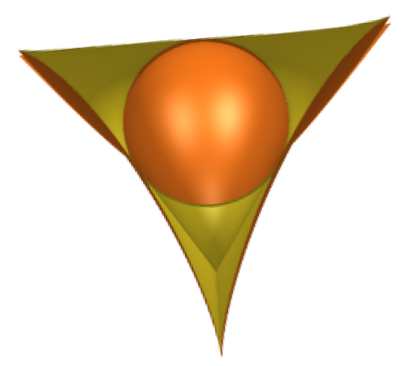}
}
\subfigure[Right triangle]{
\includegraphics[scale=0.25]{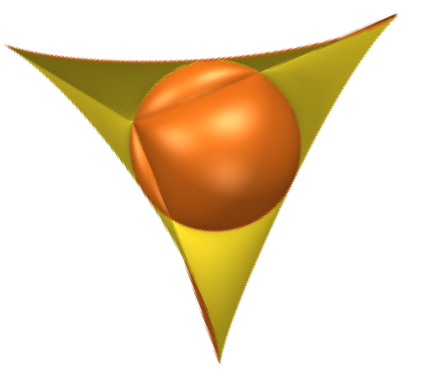}
}
\subfigure[Obtuse triangle]{
\includegraphics[scale=0.25]{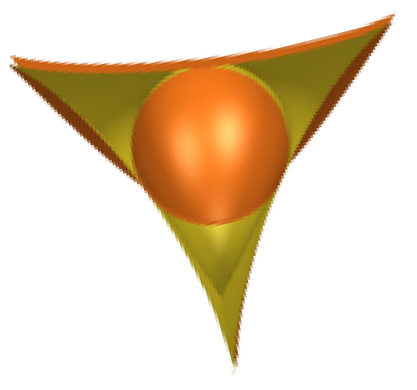}
}
\caption{The geometry of $\mathcal{V}'_{\geq2}$ for different types of triangles.}
\end{figure}






\end{document}